\documentclass{article}
\usepackage[final,main]{neurips_2026}
\usepackage[utf8]{inputenc}
\usepackage[T1]{fontenc}
\usepackage{hyperref}
\usepackage{url}
\usepackage{booktabs}
\usepackage{amsfonts}
\usepackage{amsmath}
\usepackage{amssymb}
\usepackage{amsthm}
\usepackage{mathtools}
\usepackage{nicefrac}
\usepackage{microtype}
\usepackage{xcolor}
\usepackage{graphicx}
\usepackage{subcaption}
\usepackage{enumitem}
\usepackage{algorithm}
\usepackage{algorithmic}
\usepackage[capitalize,noabbrev]{cleveref}
\theoremstyle{plain}
\newtheorem{theorem}{Theorem}[section]
\newtheorem{proposition}[theorem]{Proposition}

\theoremstyle{definition}
\newtheorem{definition}[theorem]{Definition}

\theoremstyle{remark}
\newtheorem{remark}[theorem]{Remark}

\DeclareMathOperator*{\argmin}{argmin}
\usepackage[disable,textsize=tiny]{todonotes}
\title{Surjective Pseudo-Invertible Neural Networks}
\author{
    Yamit Ehrlich \\
    Technion - Israel Institute of Technology \\
    \texttt{yamitehrlich@campus.technion.ac.il}
    \And
    Amit Arad \\
    Technion - Israel Institute of Technology
    \AND
    Nimrod Berman \\
    Ben-Gurion University
    \And
    Assaf Shocher \\
    Technion - Israel Institute of Technology
    \\
    \texttt{assaf.sh@technion.ac.il
    }
}
\begin{document}
\maketitle
\begin{abstract}
The Moore-Penrose Pseudo-inverse (PInv) is the fundamental tool for inverting linear operators. We propose a natural generalization to the non-linear regime and introduce {Surjective Pseudo-invertible Neural Networks (SPNN)}: architectures that admit a tractable non-linear PInv by construction and satisfy the corresponding geometric properties. Building on this, we formalize {Non-Linear Back-Projection (NLBP)}, the non-linear analogue of $x' = x + A^\dagger(y-Ax)$: an update that projects any sample to its closest state consistent with $f(x)=y$. Diffusion-based null-space projection revolutionized zero-shot solving of linear inverse problems via closed-form back-projection; NLBP extends this paradigm to non-linear learned ``degradations'' in the broad sense, spanning semantic mappings such as classification and object detection. The result is zero-shot inversion of complex degradations and precise semantic control over generative outputs without retraining the prior.
\end{abstract}
\section{Introduction}
\label{sec:introduction}
The generalized inverse is foundational to signal processing and data analysis. In the linear regime, the Moore-Penrose pseudo-inverse (PInv) is the canonical solution for inverting singular operators, defined by the four Penrose Identities~\cite{penrose1955generalized}. Uniquely, it guarantees the \textbf{minimal-norm solution} among all valid inverses while providing the \textbf{optimal least-squares} approximation when no exact solution exists, and it enables exact consistency enforcement via null-space back-projection.
In contrast, ``inversion'' in modern deep learning is typically approached as a regression task (autoencoders) or a probabilistic generation task (conditional diffusion). These approaches lack the structural guarantees of their linear counterparts, failing to satisfy reflexive consistency ($f(f^{-1}(f(x))) \neq f(x)$). Architectures that do enforce rigorous invertibility, such as Invertible Neural Networks~\cite{ardizzone2019analyzing} and Normalizing Flows~\cite{rezende2015variational,dinh2017density}, are constrained by strict bijectivity: they must preserve dimensionality, rendering them unsuitable for tasks where information is discarded, such as classification or detection.
We propose a natural generalization of PInv to the non-linear regime by strictly maintaining the first two Penrose identities ($gg^\dagger g = g$, $g^\dagger g g^\dagger = g^\dagger$), ensuring reflexive consistency. The remaining two require a linear adjoint with no canonical non-linear extension; we relax them rather than substituting them. This leaves the question of which generalized inverse to single out as the PInv. Unlike Gofer \& Gilboa, who pick the minimum-norm right-inverse $\|g^\dagger(y)\|$, we pick the solution induced by a Bijective Completion of the operator. Both coincide in the linear regime but diverge non-linearly; we argue in \S\ref{sec:nonlinear_pinv} that ours has natural properties (coordinate consistency, orthogonal back-projection).
We introduce \textbf{Surjective Pseudo-invertible Neural Networks (SPNN)}: surjective architectures where dimensionality is reduced and a \textbf{built-in Pseudo-Inverse} is provided by construction via a coupling mechanism whose residual degrees of freedom are resolved by a learned auxiliary network. We leverage SPNNs to introduce \textbf{Non-Linear Back-Projection (NLBP)}, enabling semantic control over generative outputs without retraining.
Our main contributions are:
\begin{itemize}[leftmargin=*,itemsep=0pt,topsep=2pt]
    \item \textbf{Non-linear Pseudo-inverse:} a natural definition justified by coordinate consistency.
    \item \textbf{SPNN:} to the best of our knowledge, the first deep architecture admitting a tractable, unique PInv satisfying this definition and the Penrose identities by construction.
    \item \textbf{Non-Linear Back-Projection (NLBP):} we generalize Iterative Back-Projection to the non-linear regime, enabling zero-shot solution of complex inverse problems.
    \item \textbf{SPNN-Linearizer:} we relax the Linearizer~\citep{berman2025said} to surjective $g$, yielding a dimension-reducing one-step flow-matching generator with improved FID.
\end{itemize}
\section{Related work}
\label{sec:related_work}
\paragraph{Deep invertible architectures.}
Modern invertible networks trace back to the Feistel cipher~\cite{feistel1973cryptography}, revitalized by NICE~\cite{dinh2014nice} (additive coupling), RealNVP~\cite{dinh2017density} (affine coupling), and Glow~\cite{kingma2018glow} (invertible $1{\times}1$ convs, often Cayley-parameterized~\cite{trockman2021orthogonalizing}). Normalizing flows~\cite{rezende2015variational} use these for density estimation; invertible nets have also been applied to inverse problems~\cite{ardizzone2019analyzing}.
\paragraph{Surjective and pseudo-invertible architectures.} 
Several lines relax strict bijectivity while retaining structured reverse 
mappings. SurVAE flows~\cite{nielsen2020survae} introduce surjective layers 
under a stochastic ELBO; PIE~\cite{beitler2019pie} proposes a pseudo-bijective 
likelihood-based autoencoder for dimensionality-reducing compression; 
PsI-NN~\cite{bolluyt2024psinn} learns an approximate inverse via auxiliary 
regression. SPNN differs by singling out the \emph{natural} PInv via a 
bijective completion (the non-linear analogue of the last two Penrose 
identities) and by deploying it for zero-shot inverse problems.
\paragraph{Generalized inversion of non-linear operators.}
Gofer and Gilboa~\cite{gofer2023generalized} define non-linear pseudo-inversion via the \emph{Best Approximate Solution} property (minimal-norm pre-image). Both definitions reduce to the linear PInv but encode different non-linear geometric choices.
\paragraph{Zero-shot inverse problems.}
Pretrained diffusion priors are SOTA for $y = \mathcal{D}(x)$. \emph{Null-space methods (linear):} DDRM~\cite{kawar2022denoising}, DDNM~\cite{wang2023zero}, SNIPS~\cite{kawar2021snips} exploit the SVD of the degradation, projecting onto the measurement subspace; these are stable but limited to linear operations. \emph{Gradient guidance (non-linear):} DPS~\cite{chung2023diffusion}, PnP-Diffusion~\cite{chung2022improving} backpropagate $\nabla_x \|y - \mathcal{D}(x)\|^2$; requires differentiable $\mathcal{D}$ and is often unstable. NLBP unifies these by generalizing Iterative Back-Projection~\cite{irani1991improving}, replacing the heuristic kernel with a learned, structural PInv $\mathcal{D}^\dagger$.
\section{Preliminaries}
\label{sec:preliminaries}
For a linear $A \in \mathbb{C}^{m \times n}$, the Moore-Penrose PInv $A^\dagger$ is the unique matrix satisfying the four Penrose identities:
\begin{equation}
    A A^\dagger A = A, \quad A^\dagger A A^\dagger = A^\dagger, \quad (A A^\dagger)^* = A A^\dagger, \quad (A^\dagger A)^* = A^\dagger A.
\label{eq:penrose}
\end{equation}
The first two yield \emph{reflexive consistency}; the last two make $AA^\dagger$ and $A^\dagger A$ orthogonal projections, so $x^* = A^\dagger y$ is the minimal-norm least-squares solution. From these one derives the linear back-projection update
\begin{equation}
    x' = x + A^\dagger(y - Ax)
    \label{eq:linear_projection}
\end{equation}
the cornerstone of null-space methods. Its iterative version IBP~\cite{irani1991improving} reads
$x_{k+1} = x_k + \lambda H (y - A x_k)$, where $H \approx A^\dagger$ and $\lambda$ is a step size.
\section{Non-linear pseudo-inverse}
\label{sec:nonlinear_pinv}
In the linear regime the Moore-Penrose PInv is the unique operator satisfying the four Penrose identities. Extending this non-linearly is non-trivial: the last two identities require a linear adjoint with no canonical non-linear extension. We propose a definition based on \emph{bijective completion}.
Throughout, we work with $\mathcal{X} = \mathbb{R}^D$ and $\mathcal{Y} = \mathbb{R}^d$ equipped with the standard Euclidean structure, and with $\mathcal{Z} = \mathbb{R}^{D-d}$ as the latent coordinate space.
\begin{definition}[Bijective completion]
\label{def:bij_completion}
Let $g: \mathcal{X} \to \mathcal{Y}$ be a smooth surjective \emph{submersion} (i.e., $Dg(x)$ has full rank $d$ for every $x \in \mathcal{X}$). A \textbf{bijective completion} of $g$ is a diffeomorphism $G: \mathcal{X} \to \mathcal{Y} \times \mathcal{Z}$ of the form
\begin{equation}
G(x) = \begin{bmatrix} g(x) \\ q(x) \end{bmatrix},
\end{equation}
where $q: \mathcal{X} \to \mathcal{Z}$ is an auxiliary smooth map. We fix the convention $q(0) = 0$, which can always be arranged by translating $q$ alone (leaving $g$ untouched). The choice of origin as reference point in Definition~\ref{def:natural_pinv} is therefore a convention rather than a substantive assumption: minimizing $\|G(x) - G(x_0)\|$ for any fixed $x_0$ defines an equivalent framework, recovered by translating $q$.
\end{definition}
\begin{remark}[Existence and non-uniqueness of $G$]
\label{rem:existence}
Our SPNN provides an explicit elementary completion (Section~\ref{sec:method}); no topological machinery is needed for the construction we use. More generally, when $g$ is a smooth surjective submersion, the level sets $g^{-1}(y)$ are smooth $(D{-}d)$-manifolds (submersion theorem) and form a locally trivial fiber bundle over $\mathcal{Y}$; for contractible $\mathcal{Y}$ the bundle is globally trivializable, guaranteeing existence of $G$. $G$ is \emph{not} unique: many choices of $q$ yield valid completions, each inducing its own pseudo-inverse, and this freedom is the non-linear analogue of choosing a basis for the null-space complement in the linear case.
\end{remark}
\begin{definition}[Natural non-linear pseudo-inverse]
\label{def:natural_pinv}
Given a surjective submersion $g$ and a bijective completion $G$, the \textbf{natural non-linear pseudo-inverse induced by $G$}, denoted $g^\dagger$, selects, for each $y$, the pre-image of $y$ whose $G$-image has minimal norm:
\begin{equation}
\label{eq:nonlin_pinv}
    g^\dagger(y) = \argmin_{x \in g^{-1}(y)} \| G(x) \|_2.
\end{equation}
We assume the minimum is attained on every fiber. In the linear case $g = A$ with $G = [A; A^\perp]$, this reduces exactly to $A^\dagger$.
\end{definition}
\subsection{Theoretical justification}
\label{subsec:justification}
We justify this definition by demonstrating two properties shared with the linear pseudo-inverse, plus a structural property (orthogonal projection) that genuinely distinguishes our definition.
\paragraph{1.\ Reflexive consistency.}
Identity $g g^\dagger g = g$ follows from the pre-image constraint: for any $x$, $g^\dagger(g(x)) \in g^{-1}(g(x))$, hence $g(g^\dagger(g(x))) = g(x)$. Identity $g^\dagger g g^\dagger = g^\dagger$ follows immediately from the right-inverse property $g g^\dagger = \mathrm{Id}$. This aligns with~\cite{gofer2023generalized}.
\paragraph{2.\ Coordinate consistency.}
Any reasonable extension of the Moore--Penrose inverse should commute with diffeomorphic re-parameterization: if $g = A \circ \phi$ for a linear surjection $A$ and a diffeomorphism $\phi$ with $\phi(0)=0$, the inverse should reduce to the linear case in the $\phi$-coordinates, namely $g^\dagger(y) = \phi^{-1}(A^\dagger y)$. Our definition satisfies this criterion: the natural completion is $G(x) = \phi(x)$, so minimizing $\|G(x)\|$ over the pre-image yields exactly $\phi^{-1}(A^\dagger y)$. The minimum-$x$-norm definition of \cite{gofer2023generalized} encodes a different geometric choice (minimization in the ambient coordinates rather than in the coordinates in which $g$ is linear) and consequently does \emph{not} reduce to the linear case under non-trivial $\phi$.
\paragraph{3.\ Orthogonal back-projection.}
The linear back-projection update $x' = x + A^\dagger(y - Ax)$ provides a closed form for the closest vector that satisfies a linear consistency constraint. A non-linear PInv should admit an equivalent closed form. We propose such a form, \emph{Non-Linear Back-Projection (NLBP)}, and show that it is doubly robust: it converges to a valid pre-image for \emph{any} reflexive right-inverse $g^\dagger$ (Proposition~\ref{prop:consistency}), but it acts as an \emph{orthogonal projection onto the nearest pre-image} only for the specific choice of $g^\dagger$ in Definition~\ref{def:natural_pinv} (Proposition~\ref{prop:orthogonal}). The latter is the genuine extra structure our definition buys over Gofer \& Gilboa.
\begin{definition}[Non-Linear Back-Projection]
Let $g$ be a surjective submersion with bijective completion $G$. The NLBP of $x$ towards $y$ is
\begin{equation}
    x' = G^{-1}\!\left( G(x) - G(g^\dagger(g(x))) + G(g^\dagger(y)) \right).
    \label{eq:nlbp_def}
\end{equation}
\end{definition}
\begin{proposition}[Consistency]
\label{prop:consistency}
For any right-inverse $g^\dagger$ satisfying the reflexive Penrose identity $g g^\dagger g = g$, the update $x'$ in Eq.~\eqref{eq:nlbp_def} lies on the pre-image of $y$: $g(x') = y$.
\end{proposition}
\begin{proof}
By construction $G(x) = [g(x), q(x)]^\top$. The $y$-component of Eq.~\eqref{eq:nlbp_def} is
$g(x') = g(x) - g(g^\dagger(g(x))) + g(g^\dagger(y))$.
By the first identity in~\eqref{eq:penrose} applied to $x$, the middle term equals $g(x)$. By surjectivity (right-inverse property of $g^\dagger$) the third term equals $y$. Hence $g(x') = y$.
\end{proof}
\begin{proposition}[Orthogonal projection]
\label{prop:orthogonal}
Assume $g^\dagger$ is the natural pseudo-inverse of Definition~\ref{def:natural_pinv}, and that $q \circ g^\dagger$ is constant on $\mathcal{Y}$ (so the minimizer's $q$-coordinate is independent of $y$; this holds for SPNN by construction, see Proposition~\ref{prop:spnn_constancy}). Then the NLBP update $x'$ is the orthogonal projection (in the $G$-metric) of $x$ onto the solution manifold $g^{-1}(y)$.
\end{proposition}
\begin{proof}
Let $\Delta G = G(x') - G(x)$. From Eq.~\eqref{eq:nlbp_def}, the range component is $g(x') - g(x) = y - g(x)$ and the null-space component is $q(x') - q(x) = q(g^\dagger(y)) - q(g^\dagger(g(x)))$. By the constancy assumption, $q(g^\dagger(y)) = q(g^\dagger(g(x)))$, so the null-space component vanishes: $\Delta G = [y - g(x), 0]^\top$. Any tangent vector $v$ to the solution manifold satisfies $dG(v) = [0, dq]^\top$, hence $\langle \Delta G, dG(v) \rangle = 0$.
\end{proof}
\begin{proposition}[Exact constancy for SPNN]
\label{prop:spnn_constancy}
Let $g$ be an SPNN block with bijective completion $G(x) = [g(x);\, x_1]$ (i.e., $q = x_1$). Then the natural pseudo-inverse of Definition~\ref{def:natural_pinv} satisfies $q \circ g^\dagger \equiv 0$. Consequently, the constancy hypothesis of Proposition~\ref{prop:orthogonal} holds, and NLBP acts as exact orthogonal projection (in the $G$-metric) onto the solution manifold $g^{-1}(y)$.
\end{proposition}
\begin{proof}
On each fiber $g^{-1}(y)$ we have $\|G(x)\|^2 = \|y\|^2 + \|x_1\|^2$, which separates into a fiber-fixed term $\|y\|^2$ and a term $\|x_1\|^2$ that is strictly convex in the unconstrained parameterization of the fiber by $x_1$. The minimum is at $x_1 = 0$ for every $y$, so $q \circ g^\dagger \equiv 0$.
\end{proof}
\begin{remark}[Multi-block case]
\label{rem:multiblock}
For a single block, $x_1 = 0$ shows any constant $r$ would suffice. In a multi-block SPNN, channels are mixed between blocks by the learned orthogonal $U$, so that each block introduces new discraded coordinates. The joint minimizer across the composition is then $y$-dependent, making $r$ non-trivial: \S\ref{subsec:ablation} verifies that $r \equiv 0$ retains $g g^\dagger = \mathrm{Id}$ but breaks NLBP. More generally, Proposition~\ref{prop:orthogonal}'s constancy assumption holds whenever $\|G\|^2$ splits into a fiber-fixed term and a $y$-independent fiber-varying term.
\end{remark}
\section{Surjective Pseudo-invertible Neural Network (SPNN)}
\label{sec:method}
\begin{figure}[t]
    \centering
    \includegraphics[width=0.49\linewidth]{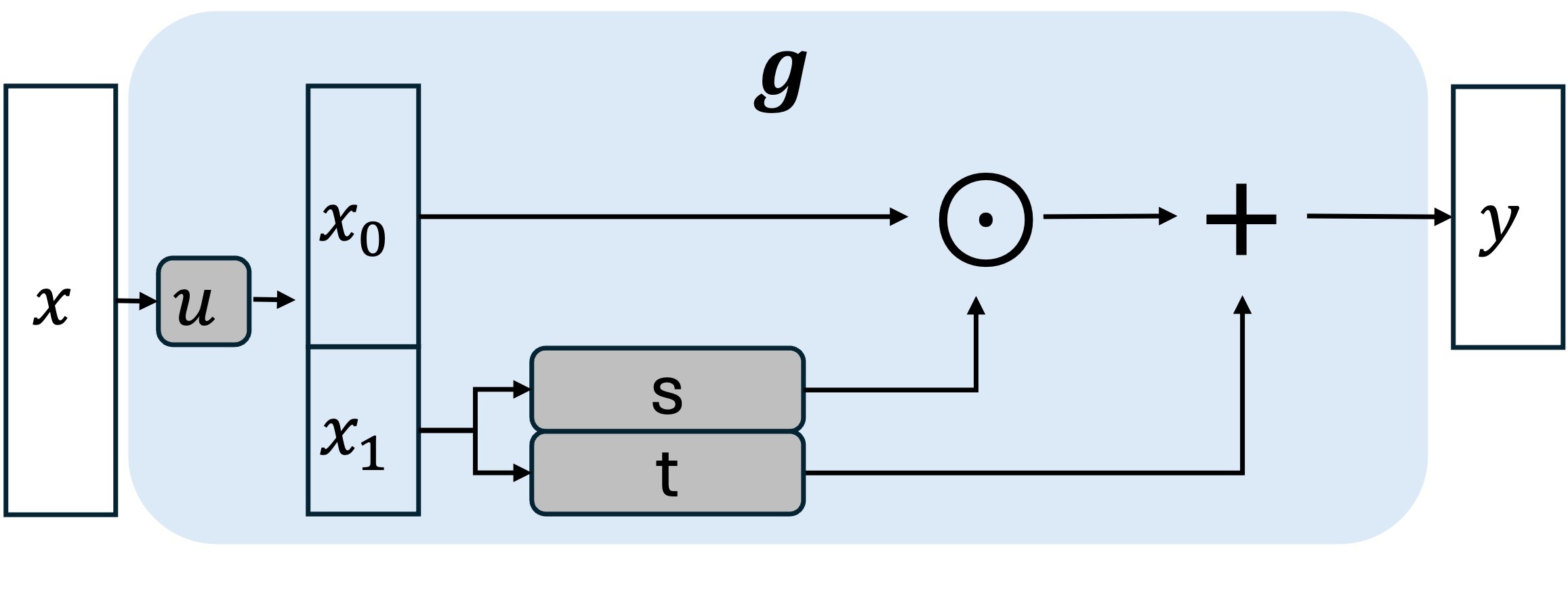}\hfill
    \includegraphics[width=0.49\linewidth]{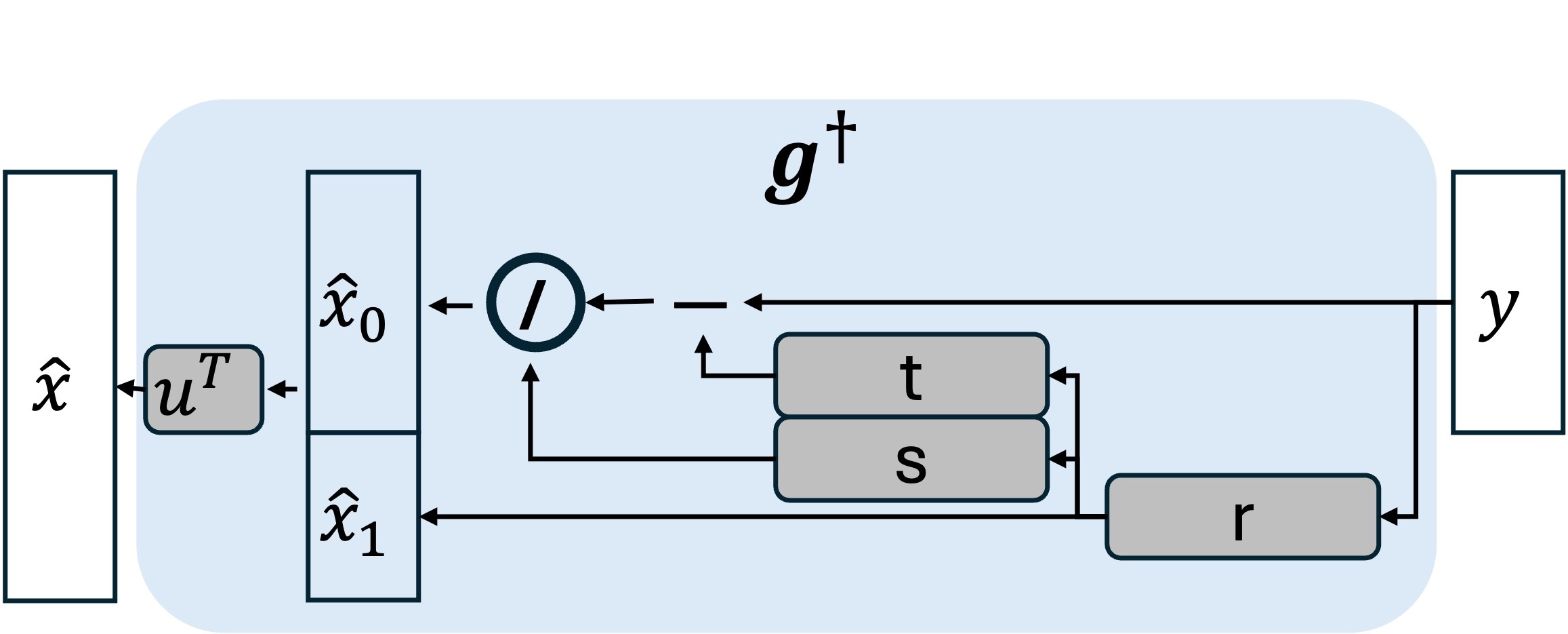}
    \caption{\small \textbf{The SPNN block.} \textit{Left:} forward $g$ rotates the input by $U$, splits it into $(x_0, x_1)$, and produces the output $y = x_0 \odot s(x_1) + t(x_1)$. \textit{Right:} pseudo-inverse $g^\dagger$ uses an auxiliary network $r$ to predict $\hat{x}_1$ from $y$; the reverse coupling flow reconstructs the pre-image $\hat{x}$.}
    \label{fig:diagram}
\end{figure}
In this section we introduce the Surjective Pseudo-invertible Neural Network (SPNN). Unlike standard Invertible Neural Networks that preserve dimensionality to ensure bijectivity, SPNNs are designed to be dimension-reducing ($d < D$) while maintaining a rigorous structural right-inverse. Note that an SPNN is constructed as a submersion ($s > 0$), and Section~\ref{sec:nonlinear_pinv}'s guarantees apply to this network $g$. When SPNN is used as a surrogate for some target degradation $\mathcal{D}$ that is not itself a submersion (e.g., a real classifier with saturating logits), the guarantees apply to the surrogate $g \approx \mathcal{D}$, not to $\mathcal{D}$ itself.
\subsection{Surjective pseudo-invertible block}
Let $x \in \mathbb{R}^D$ be the input. We partition it into two sub-vectors and define the forward operation $g(\cdot)$ as an \emph{affine coupling} where $x_1$ modulates $x_0$ via learned scale and translation:
\begin{equation}
    x \to [x_0 \,\vert\, x_1], \quad x_0 \in \mathbb{R}^d,\; x_1 \in \mathbb{R}^{D-d}, \qquad y = x_0 \odot s(x_1) + t(x_1),
    \label{eq:forward}
\end{equation}
where $s, t: \mathbb{R}^{D-d} \to \mathbb{R}^d$ are CNNs, $\odot$ is element-wise multiplication, and we enforce $s > 0$ pointwise (making the block a submersion by construction). For any $y$ and any choice of $x_1$, a unique $x_0$ satisfies~\eqref{eq:forward}; hence the block is surjective.
To define the PInv $g^\dagger(y)$ we resolve the unknown $x_1$ with a learnable auxiliary network $r: \mathbb{R}^d \to \mathbb{R}^{D-d}$. The inverse proceeds in two steps (Fig.~\ref{fig:diagram}, right):
\begin{equation}
    \hat{x}_1 = r(y), \qquad \hat{x}_0 = \frac{y - t(\hat{x}_1)}{s(\hat{x}_1)}, \qquad \hat{x} = [\hat{x}_0 \,\vert\, \hat{x}_1].
    \label{eq:inverse}
\end{equation}
\paragraph{The bijective completion of an SPNN block.}
For SPNN the natural choice is $q(x) \coloneqq x_1$ (the coordinates already discarded), which makes $G(x) = [x_0 \odot s(x_1) + t(x_1);\, x_1]$ a global diffeomorphism by construction ($|\det DG| = \prod s(x_1) > 0$), and is the non-linear analogue of $q = A^\perp$. This choice satisfies the constancy assumption of Proposition~\ref{prop:orthogonal} (see Proposition~\ref{prop:spnn_constancy}). Other $q$ are admissible by Definition~\ref{def:bij_completion}; consistency (Prop.~\ref{prop:consistency}) holds for any $q$; orthogonal projection requires $q \circ g^\dagger$ constant.
\subsection{Two-phase training strategy}
\label{subsec:training}
Training decomposes into two phases. In \textbf{Phase~I} (task learning) we optimize the forward parameters $\theta_g = \{s, t, U\}$ with a task-specific objective $\mathcal{L}_{\text{task}}$ (cross-entropy or MSE); this phase is identical to training a standard feed-forward network. In \textbf{Phase~II} we freeze $\theta_g$ and optimize only $r$ using
\begin{equation}
    \mathcal{L}_{\text{natural}} = \mathbb{E}_{y} \left[ \| G(g^\dagger(y)) \|_2^2 \right],
    \label{eq:loss_natural}
\end{equation}
so $r$ predicts $\hat{x}_1$ placing the solution closest to the origin in $G$-metric, matching Def.~\ref{def:natural_pinv}.
\begin{proposition}[Exact right-inverse and reflexive consistency]
\label{prop:right_inverse}
For any choice of networks $s, t, r$, the operator $g^\dagger$ defined by Eq.~\eqref{eq:inverse} satisfies $g \circ g^\dagger = \mathrm{Id}$, and the SPNN architecture satisfies the first two Penrose identities $g g^\dagger g = g$ and $g^\dagger g g^\dagger = g^\dagger$ by construction.
\end{proposition}
\begin{proof}
For arbitrary $y$, $\hat{x}_1 = r(y)$ and $\hat{x}_0 = (y - t(r(y)))/s(r(y))$, so $g(\hat{x}) = \hat{x}_0 \odot s(\hat{x}_1) + t(\hat{x}_1) = (y - t(r(y))) + t(r(y)) = y$. Hence $g g^\dagger = \mathrm{Id}$. The first two Penrose identities follow immediately: $g g^\dagger g = (g g^\dagger) g = g$ and $g^\dagger g g^\dagger = g^\dagger (g g^\dagger) = g^\dagger$.
\end{proof}
\begin{remark}[Why $r$ is needed]
\label{rem:r_needed}
$g \circ g^\dagger = \mathrm{Id}$ holds for \emph{any} $r$, including a constant (Proposition~\ref{prop:right_inverse}). What $r$ buys is the \emph{natural} PInv: it selects the right-inverse whose composed pre-image minimizes $\|G(g^\dagger(y))\|^2$ jointly across blocks. Because adjacent blocks share channels through $U$, this minimizer is generically $y$-dependent (Remark~\ref{rem:multiblock}); a constant $r$ produces consistent but algebraically wrong inverses that diverge under NLBP-guided diffusion (\S\ref{subsec:ablation}).
\end{remark}
\subsection{Implementation details}
\label{subsec:implementation}
The SPNN stacks \texttt{PixelUnshuffle} (factor $4$) interleaved with \texttt{ConvSPNN} coupling blocks; the exact count is task-dependent. Each \texttt{ConvSPNN} block splits channels, passes the informative half ($x_0$) forward and treats the rest ($x_1$) as null-space, preceded by a learnable orthogonal $U$ via Cayley transform. $s, t, r$ share a U-Net-style \texttt{ConvMLP}; $s>0$ is enforced via $\exp(2\tanh(\cdot))$. We also add surjectivity ($\|y - g(g^\dagger(y))\|_2^2$) and stability ($\|x - g^\dagger(g(x))\|_2^2$) losses in Phase~I (full details in App.~\ref{app:architecture}).
\section{Applications}
\label{sec:application}
\subsection{Non-Linear Back-Projection for inverse problems}
\label{subsec:nlbp_app}
SPNNs enable zero-shot inversion of \textbf{learned non-linear} ``degradations''. In this sense, degradations include semantic mappings such as classification and object detection. We treat $\mathcal{D}$ as a black box, train an SPNN $g \approx \mathcal{D}$ via Section~\ref{sec:method}, and use NLBP to drive a pre-trained diffusion prior to samples consistent with a target measurement.
Inspired by DDNM~\cite{wang2023zero}, we integrate NLBP into a pre-trained sampling loop (Fig.~\ref{fig:nlbp_pipe}). At each timestep $t$, given a clean estimate $\hat{x}_{0|t}$ and target $y$:
\begin{equation}
    \hat{x}'_{0|t} = G^{-1}\!\left( G(\hat{x}_{0|t}) + \lambda \left[ G(g^\dagger(y)) - G(g^\dagger(g(\hat{x}_{0|t}))) \right] \right),
    \label{eq:gentle_nlbp}
\end{equation}
where $\lambda \in [0, 1]$ is a guidance scale. In $G$-coordinates this becomes $g(\hat{x}'_{0|t}) = (1 - \lambda) g(\hat{x}_{0|t}) + \lambda y$ and (under Proposition~\ref{prop:orthogonal}'s constancy) $q(\hat{x}'_{0|t}) = q(\hat{x}_{0|t})$: semantic content is steered towards $y$ via linear interpolation, while null-space details from the prior remain invariant. This enables reconstruction from highly compressed inputs (e.g., class labels) without disrupting texture coherence.
Note that measurement consistency in NLBP relies only on $g g^\dagger=\mathrm{Id}$ 
(Prop.~\ref{prop:consistency}): any reflexive right-inverse keeps the range component 
of $\Delta G$ equal to $y - g(x)$. What the natural-PInv objective additionally buys 
is the vanishing of the null-space component (orthogonality of the update), guaranteeing the closest solution satisfying the constraint in $G$-metric (Prop.~\ref{prop:orthogonal}). Full sampling details in App.~\ref{app:nlbp_details}.
\begin{figure}[t]
    \centering
    \includegraphics[width=0.85\linewidth]{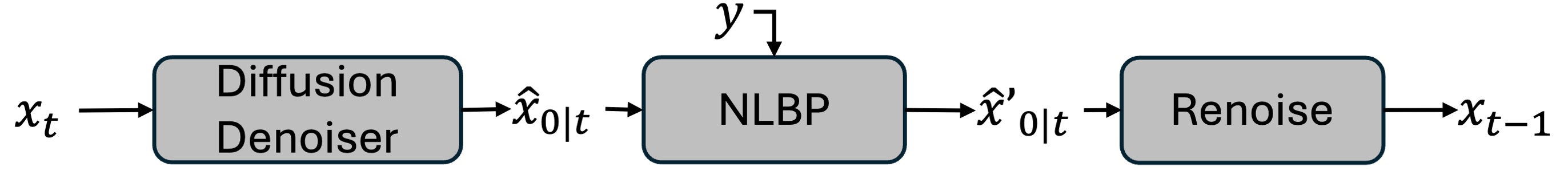}
    \caption{\small \textbf{NLBP-guided diffusion sampling step.} At each timestep, the denoiser produces $\hat{x}_{0|t}$, NLBP corrects it toward $y$, and the corrected $\hat{x}'_{0|t}$ is renoised to $x_{t-1}$.}
    \label{fig:nlbp_pipe}
\end{figure}
\begin{figure}[b]
    \centering
    \begin{subfigure}[b]{0.45\linewidth}
        \centering
        \includegraphics[width=\linewidth]{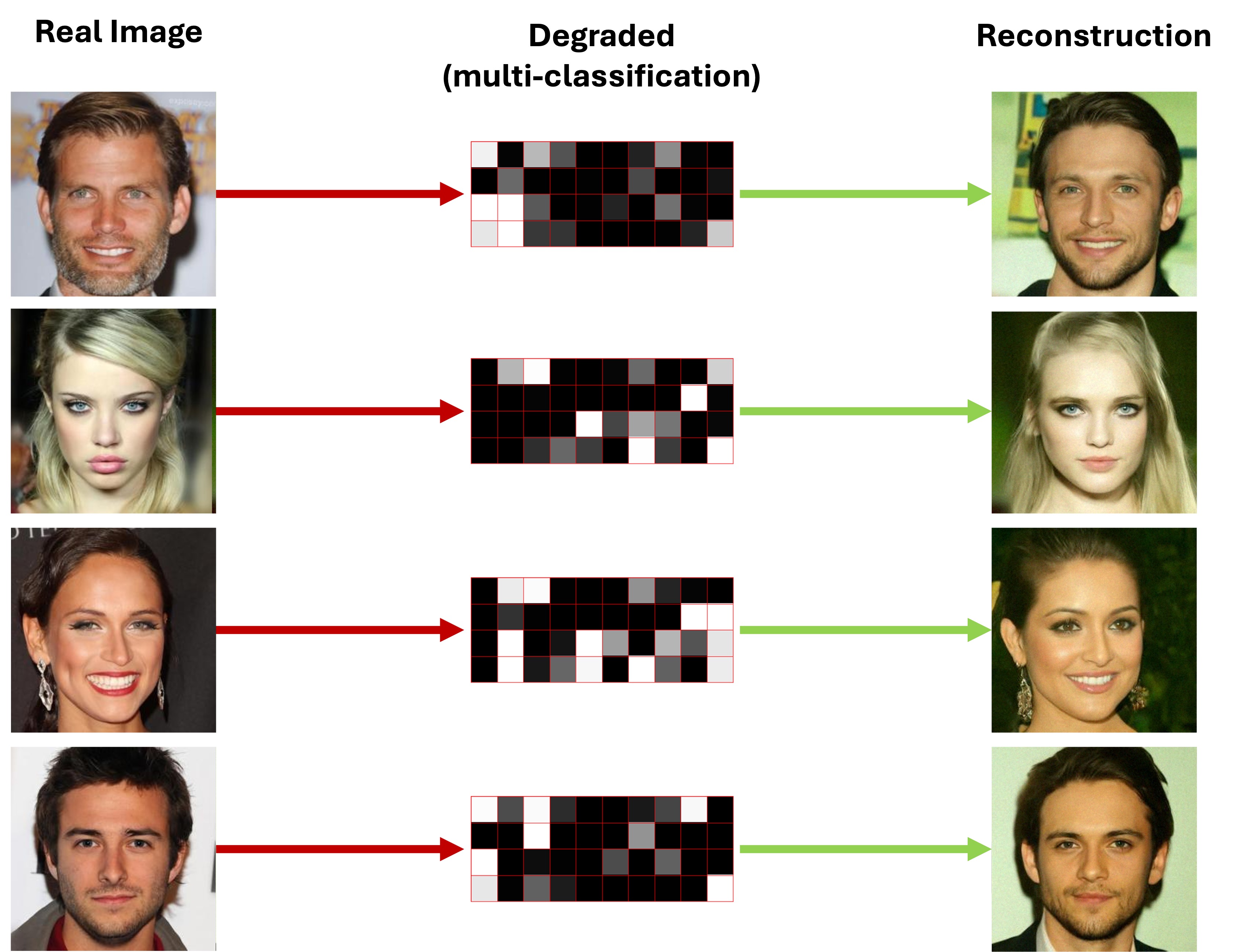}
    \end{subfigure}\hfill
    \begin{subfigure}[b]{0.45\linewidth}
        \centering
        \includegraphics[width=\linewidth]{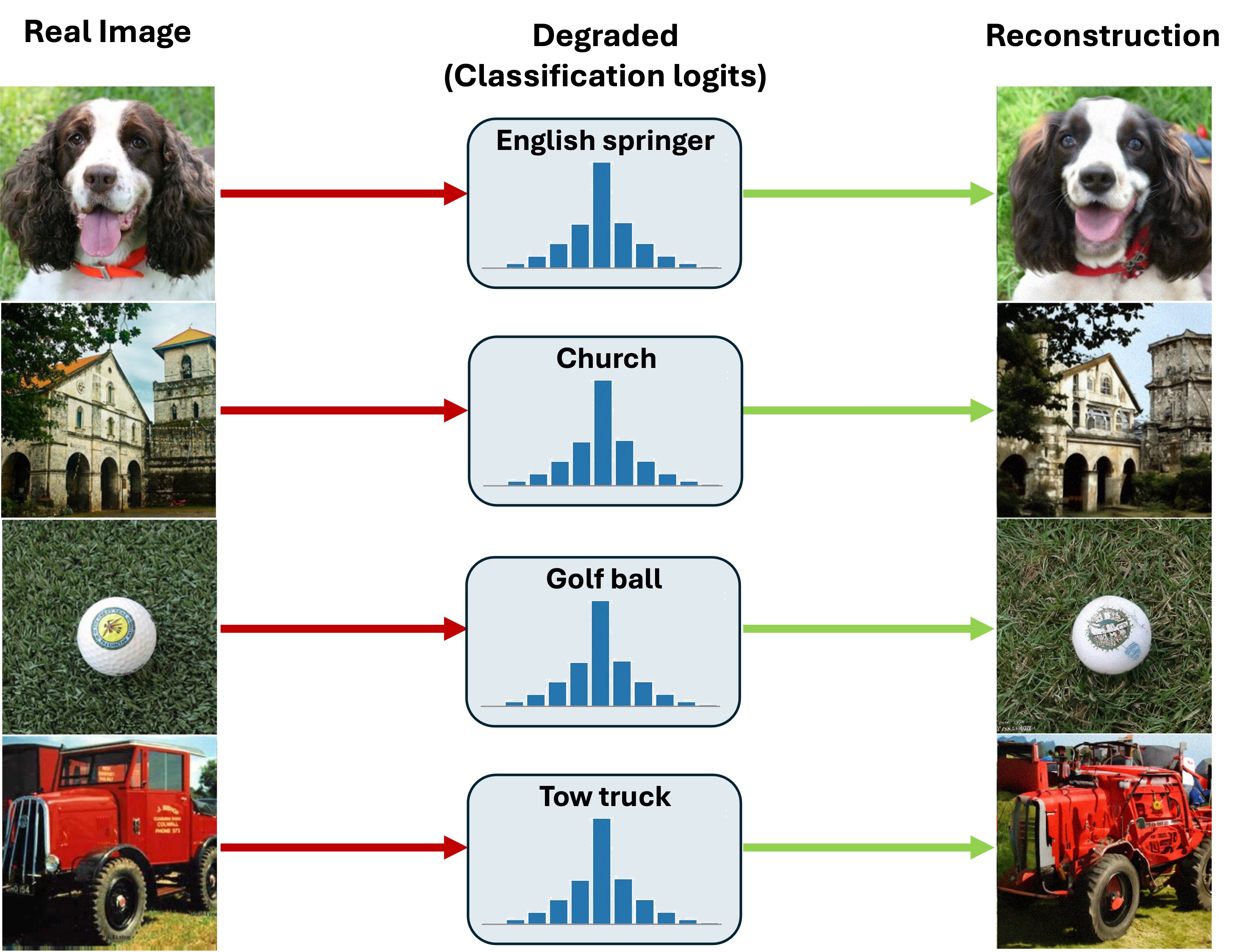}
    \end{subfigure}
    \caption{\small \textbf{Reconstruction from semantics.} Each row: clean image / measurement / NLBP-guided reconstruction. \textit{Left:} CelebA-HQ, where the measurement is a $40$-dim attribute logit vector (shown as a binary grid for illustration). \textit{Right:} ImageNet, where the measurement is the full $1000$-dim class logit vector (illustrated by the top-1 class and a stylized peaked distribution). The reconstruction is one sample from the classifier's pre-image manifold; it need not match the original pixel-wise, only in the logit measurement.}
    \label{fig:recon}
\end{figure}
\vspace{-0.4em}
\subsection{SPNN based Linearizer}
\label{subsec:linearizer_method}
The Linearizer framework~\cite{berman2025said} composes a linear core
between two invertible networks $g_x, g_y$ to obtain a map
$f(x) := g_y^{-1}(A\, g_x(x))$ that, while non-linear over standard
Euclidean vector spaces, is linear w.r.t.\ \emph{lifted} vector-space
structures on $\mathcal{X}, \mathcal{Y}$ defined by pulling vector
addition and scalar multiplication back through $g_x, g_y$. E.g., on
$\mathcal{X}$:
\begin{equation}
a \odot_x x := g_x^{-1}\!\left(a\, g_x(x)\right) ,\quad
x_1 \oplus_x x_2 :=  g_x^{-1}\!\left(g_x(x_1) + g_x(x_2)\right) .
\end{equation}
We observe that this construction depends only on
$g\, g^\dagger = \mathrm{Id}$, not on full invertibility, so it extends
to surjective $g$ by replacing $g^{-1}$ with $g^\dagger$ above. The
resulting operations make the $d$-dimensional submanifold
$\mathrm{im}(g_x^\dagger) \subset \mathcal{X}$ a vector space
canonically isomorphic to $(\mathcal{Y}, +, \cdot)$ via
$g_x$.\footnote{A vector-space structure on the \emph{full} $\mathcal{X}$
can be obtained by lifting through the bijective completion $G$
(Section~\ref{sec:nonlinear_pinv}) instead of $g^\dagger$, i.e.,
$x_1 \oplus_x^G x_2 := G_x^{-1}(G_x(x_1)+G_x(x_2))$. We do not pursue
this route: the Linearizer property below uses only
$g\, g^\dagger = \mathrm{Id}$, $f$ takes values in $\mathrm{im}(g_y^\dagger)$
regardless of input, and the slice formulation yields a $d$-dimensional
vector space matching the dimension of $A$.}
SPNNs satisfy the right-inverse property by construction
(Proposition~\ref{prop:right_inverse}), with the additional benefit
that $A$ acts on a low-dimensional latent rather than on the full
input space.
\begin{proposition}[SPNN-Linearizer is linear in the lifted operations]
\label{prop:spnn_linearizer}
Let $g_x, g_y$ be SPNNs and let $A$ be a linear operator.
$f(x) := g_y^\dagger(A\, g_x(x))$ is linear over the lifted vector-spaces.
\end{proposition}
\begin{proof}\vspace{-0.3cm}
\begin{align*}
f(a_1 \odot_x x_1 \oplus_x a_2 \odot_x x_2)
&= g_y^\dagger\!\left(A\, g_x\!\left(g_x^\dagger(a_1 g_x(x_1) + a_2 g_x(x_2))\right)\right) \\
&= g_y^\dagger\!\left(A\,(a_1 g_x(x_1) + a_2 g_x(x_2))\right) \\
&= g_y^\dagger\!\left(g_y g_y^\dagger\!\left(a_1 A g_x(x_1) + a_2 A g_x(x_2)\right)\right) \\
&= a_1 \odot_y f(x_1) \oplus_y a_2 \odot_y f(x_2).\qedhere
\end{align*}
\vspace{-0.3cm}
\end{proof}
\vspace{-0.3cm}
\paragraph{Iteration collapse.}
When $g_x = g_y = g$ and we iterate $f$ with linear cores
$A_1, A_2, \ldots, A_N$ (which may differ per step, e.g.,\ a
time-varying flow operator $A_t$), every intermediate
$g\, g^\dagger$ telescopes to the identity, so the whole composition
reduces to one matrix product wrapped by a single $g$/$g^\dagger$ pair:
\begin{equation}
f_{A_N} \circ \cdots \circ f_{A_1}(x) = g^\dagger(B\, g(x)),
\qquad
B := A_N \cdots A_1.
\end{equation}
The $N$-step flow can therefore be precomputed once as a single matrix
$B$ and applied with one forward/inverse SPNN pass. This is the basis
of one-step flow-matching distillation, the main application
demonstrated by~\cite{berman2025said}, now extended to
dimension-reducing backbones. We instantiate this in
\S\ref{subsec:linearizer_exp}.

\vspace{-2mm}
\section{Experiments}
\vspace{-2mm}
\label{sec:experiments}
We evaluate the proposed framework on three settings: (i) \textbf{semantic restoration} from $40$-dim attribute logits on CelebA-HQ~\cite{karras2018celebahq} and from $1000$-dim class logits on ImageNet~\cite{russakovsky2015imagenet} (both at $256{\times}256$, classifier as the forward operator); (ii) \textbf{compositional generation} from object detection on Pascal~VOC~2007~\cite{everingham2010pascal} (\S\ref{subsec:detection_app}, CenterNet-style detector~\cite{zhou2019objects} as the forward operator); and (iii) \textbf{one-step flow-matching generation} on CelebA $64{\times}64$ via the SPNN-Linearizer (\S\ref{subsec:linearizer_exp}). The diffusion prior is an unconditional DDPM~\cite{dhariwal2021diffusion} for the classification settings and a class-conditional ADM~\cite{dhariwal2021diffusion} for the detection setting; architecture details in Appendix~\ref{app:architecture}.
\vspace{-0.4em}
\subsection{Semantic reconstruction}
\label{subsec:reconstruction}
We reconstruct an image given \emph{only} its classification logit vector. We extract $y_{\text{GT}} = g(x_{\text{orig}})$ from an unseen test image and run NLBP-guided diffusion (Eq.~\ref{eq:gentle_nlbp}) with $y_{\text{GT}}$ fixed. Fig.~\ref{fig:recon} shows CelebA-HQ reconstructions matching 40-dim attribute logits (left) and ImageNet matching $1000$-dim class logits (right). The reconstruction is one sample from the classifier's pre-image manifold.
To quantify fidelity on CelebA-HQ, we computed reconstruction accuracy over $100$ test samples (Fig.~\ref{fig:attrib_analysis}). Mean binary agreement is $92.3\%$ across all 40 attributes. Structural attributes (\textit{Eyeglasses}, \textit{Hats}, \textit{Neckties}) reconstruct near-perfectly ($>97\%$). Error concentrates in ambiguous features (\textit{High Cheekbones}, \textit{Arched Eyebrows}), where the classifier's own uncertainty is also higher.
\begin{table}[b]
\vspace{-4mm}
\centering
\caption{\small \textbf{Zero-shot semantic reconstruction on CelebA-HQ ($6{,}000$ samples).} Lower is better for FID/KID/Err; higher for Agree. ``self''/``cross'': method's own / other's classifier (SPNN, ResNet); ``third-party'': independently-trained EfficientNet. $z$: DPS step size. Comparison to text-to-image priors (SD~1.5, SDXL) is in App.~\ref{app:t2i_comparison}.}
\label{tab:dps_comparison}
\footnotesize
\setlength{\tabcolsep}{3pt}
\begin{tabular}{llcccccccc}
\toprule
Method & Guidance & FID$\downarrow$ & KID$\downarrow$ & Agree self$\uparrow$ & Agree cross$\uparrow$ & Agree 3rd$\uparrow$ & Err self$\downarrow$ & Err cross$\downarrow$ & s/img \\
\midrule
NLBP+SPNN (ours) & NLBP          & \textbf{45.1} & \textbf{.033} & \textbf{92.5\%} & \textbf{88.7\%} & \textbf{87.9\%} & \textbf{0.073} & \textbf{0.123} & 5.7 \\
DPS+ResNet       & DPS ($z{=}1.0$) & 185.4 & .166 & 92.2\% & 73.8\% & 73.0\% & 0.077 & 0.273 & 1.7 \\
DPS+ResNet       & DPS ($z{=}0.5$) & 101.4 & .071 & 91.7\% & 79.8\% & 80.0\% & 0.083 & 0.200 & 1.7 \\
DPS+ResNet       & DPS ($z{=}0.1$) & 120.0 & .085 & 90.4\% & 81.3\% & 83.7\% & 0.098 & 0.186 & 1.7 \\
\bottomrule
\end{tabular}
\end{table}
\begin{table}[h]
\vspace{-0.5cm}
\centering
\begin{minipage}[c]{0.25\linewidth}
\footnotesize
\centering
\setlength{\tabcolsep}{4pt}
\begin{tabular}{lccc}
\toprule
Method & FID$\downarrow$ & KID$\downarrow$ & Agree$\uparrow$ \\
\midrule
NLBP (ours)        & \textbf{30.65} & \textbf{.021} & 49\% \\
DPS ($z{=}1.0$)  & 54.00          & .034          & 71.1\% \\
DPS ($z{=}0.5$)  & 50.42          & .033          & \textbf{71.6\%} \\
DPS ($z{=}0.1$)  & 81.78          & .068          & 38.8\% \\
\bottomrule
\end{tabular}
\end{minipage}\hfill
\begin{minipage}[c]{0.2\linewidth}
\centering
\includegraphics[width=0.8\linewidth]{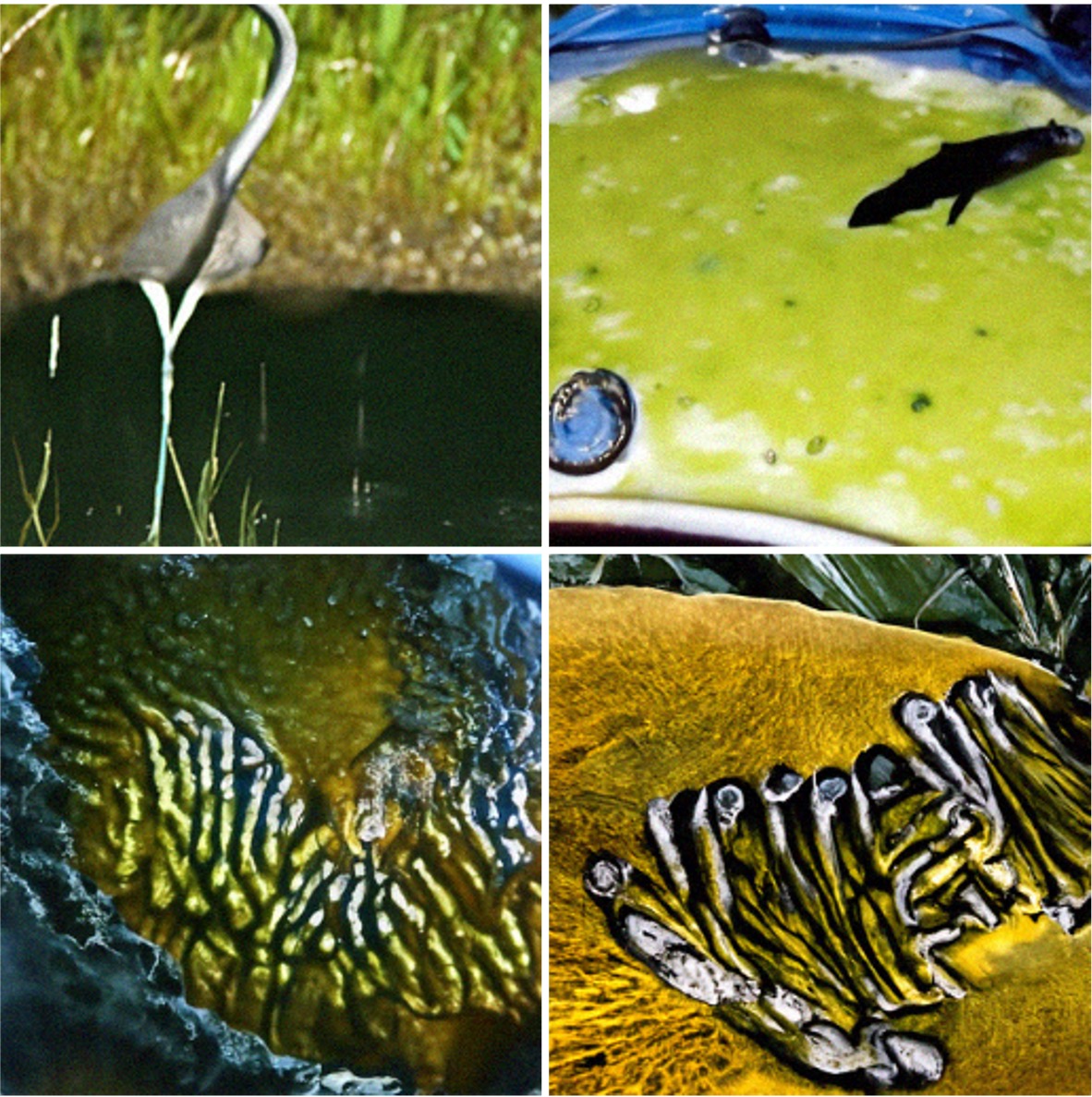}
\end{minipage}
\begin{minipage}[c]{0.44\linewidth}
\caption{\small \textbf{Zero-shot reconstruction on ImageNet} ($1000$-dim logits, $256{\times}256$; setup as Table~\ref{tab:dps_comparison}). Agree: overlap with target's top-5 under an independent ConvNeXt. DPS's apparent advantage is from adversarial behavior. Center: figure are DPS outputs matching ImageNet class 0: tench, \textit{Tinca tinca} (Usually humans holding fish).}
\label{tab:dps_imagenet}
\end{minipage}
\end{table}
\subsection{Ablation: the role of the auxiliary network $r$}
\label{subsec:ablation}
\begin{figure}[t]
    \centering
    \begin{minipage}[t]{0.27\linewidth}
        \centering
        \includegraphics[width=\linewidth]{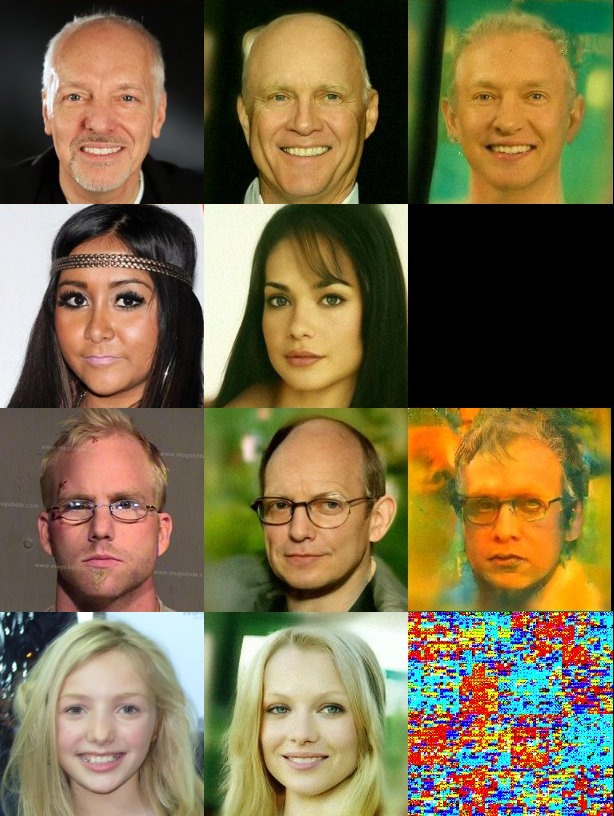}
        \captionof{figure}{\textbf{Ablation of $r$.} With $r \equiv 0$ the NLBP-guided sampler diverges; with trained $r$ it does not.}
        \label{fig:ablation_r}
    \end{minipage}\hfill
    \begin{minipage}[t]{0.69\linewidth}
        \centering
        \includegraphics[width=0.9\linewidth]{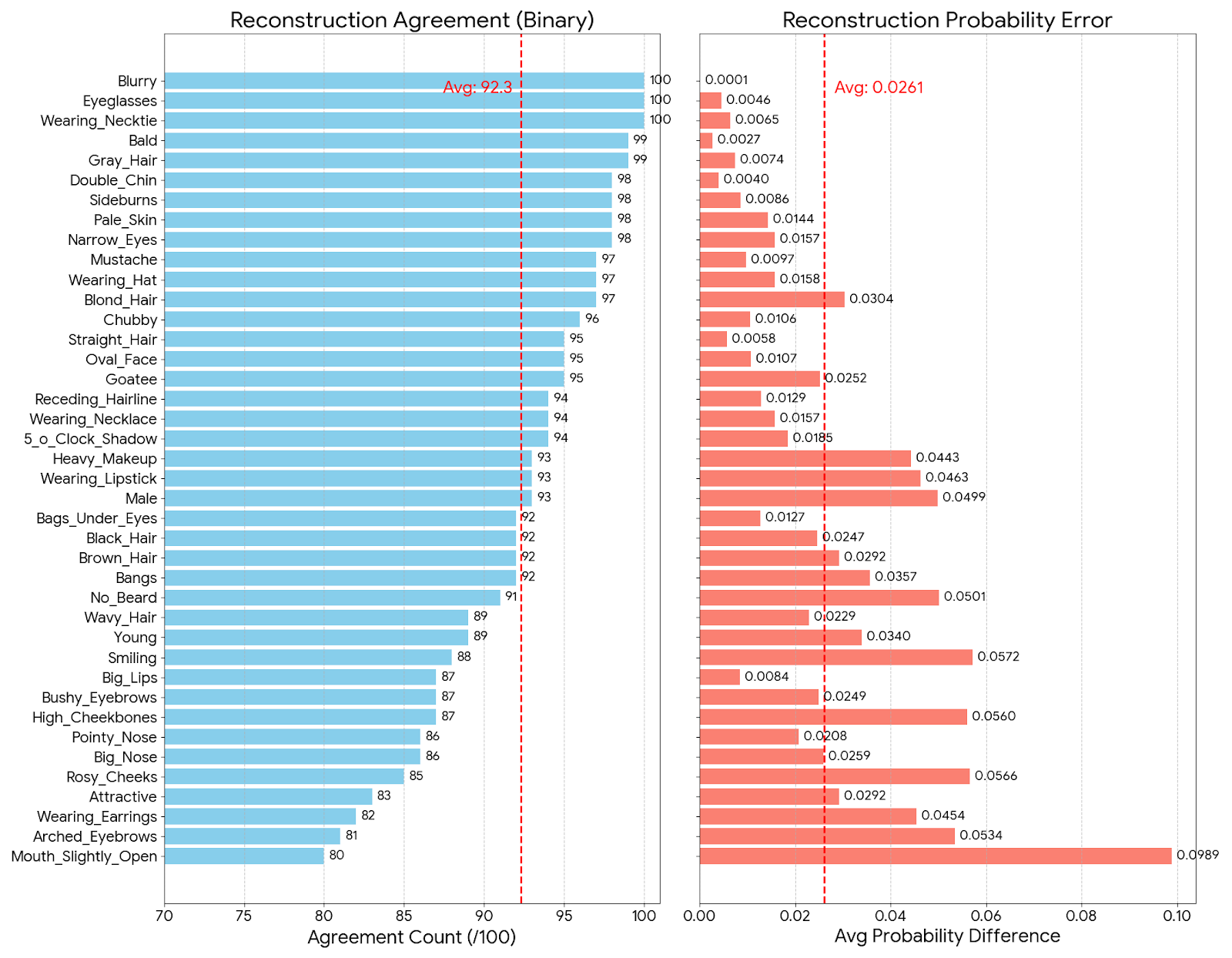}
        \captionof{figure}{\textbf{Per-attribute reconstruction fidelity on CelebA-HQ.} Per-attribute binary agreement (left bars) and mean absolute error in probability space (right bars) on $100$ test samples. Mean binary agreement is $92.3\%$ across all $40$ attributes. Numerical values in Appendix~\ref{app:attribs}.}
        \label{fig:attrib_analysis}
    \end{minipage}
    \vspace{-5mm}
\end{figure}
\vspace{-0.4em}
We replace $r$ with the constant predictor $r \equiv 0$, the trivial single-block minimizer (Proposition~\ref{prop:spnn_constancy}). This preserves the right-inverse identity $g g^\dagger = \mathrm{Id}$ but is no longer the joint minimizer of $\|G(g^\dagger(y))\|^2$ across the multi-block composition (Remark~\ref{rem:multiblock}). All other components ($g$, NLBP sampler, diffusion prior) are kept identical.
Fig.~\ref{fig:ablation_r} shows the result on four test inputs. With $r \equiv 0$, the sampler diverges in different ways across inputs (color drift, black-out, high-frequency noise), even though measurement consistency holds throughout. What fails is the orthogonality of the NLBP step (Proposition~\ref{prop:orthogonal}). The natural-PInv objective, not the right-inverse identity, is what makes NLBP stable.

\subsection{Comparison to zero-shot baselines}\vspace{-0.2cm}
\label{subsec:baselines}
We compare against DPS~\cite{chung2023diffusion}, as a zero-shot baseline with pre-trained diffusion. For DPS we train a ResNet classifier matched to our SPNN; both use the same pre-trained DDPM~\cite{dhariwal2021diffusion}. To avoid adversarial self-evaluation, agreement is reported against three classifiers: the method's own (``self''), the other's (``cross''), and an independent EfficientNet (``third-party''). All over $6{,}000$ samples. We additionally include a (non-comparable) reference comparison against large text-to-image priors (Stable Diffusion 1.5~\cite{rombach2022ldm}, SDXL~\cite{podell2023sdxl}) in App.~\ref{app:t2i_comparison}.
Results are in Tables~\ref{tab:dps_comparison} (CelebA-HQ) and~\ref{tab:dps_imagenet} (ImageNet). Our method beats DPS by a large FID/KID margin on both, and on CelebA-HQ in agreement under all three classifiers. DPS does win on ImageNet top-5 agreement, but the center of Table~\ref{tab:dps_imagenet}, and FID results reveal the cause: DPS outputs classify into the target class without depicting it; the same adversarial pattern appears on CelebA-HQ as a self-vs-independent gap ($\sim\!91\%$ vs.\ $\sim\!73$--$84\%$). FID/KID, which measures distributional fidelity, confirms. Text-to-image priors (App.~\ref{app:t2i_comparison}) are outperformed too.
\begin{figure}[]
    \centering
    \begin{subfigure}[t]{0.39\linewidth}
      \includegraphics[width=\linewidth]{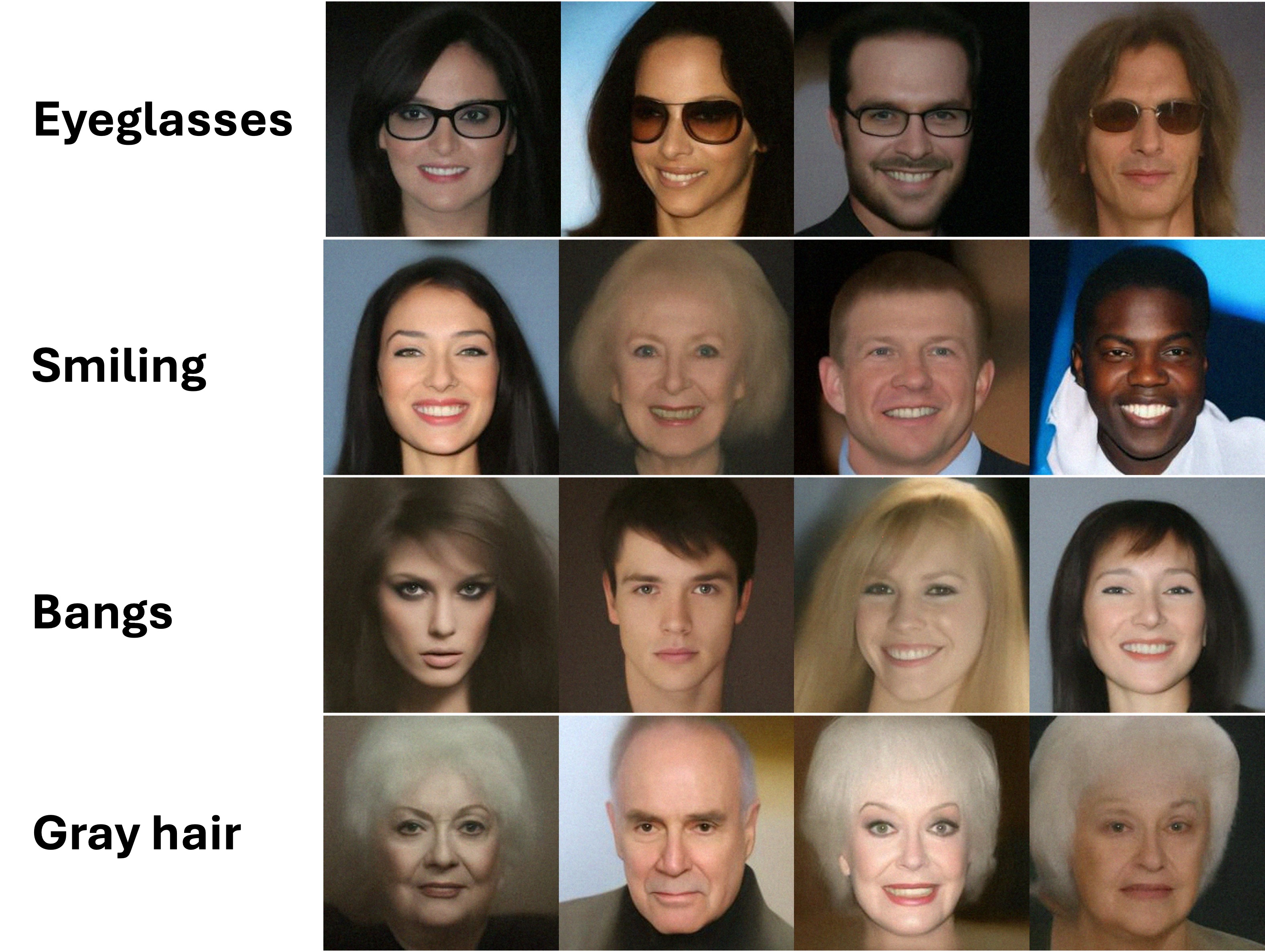}
    \end{subfigure}\hfill
    \begin{subfigure}[t]{0.42\linewidth}
      \includegraphics[width=\linewidth]{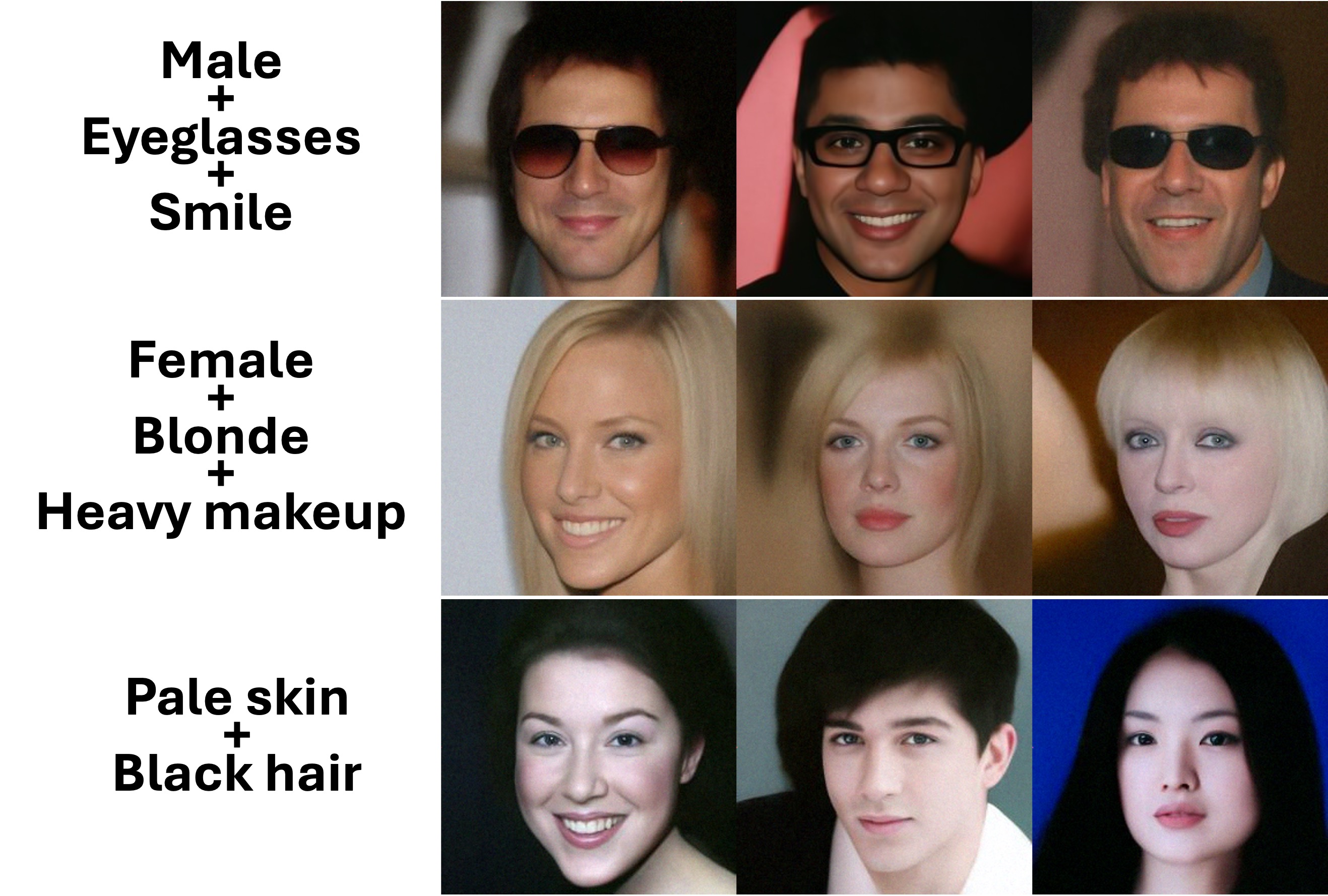}
    \end{subfigure}
    \caption{\small \textbf{Attribute-conditioned generation.} \textit{Left:} enforcing one attribute (e.g.,\ ``Eyeglasses'') produces diverse samples adhering to that constraint while freely hallucinating the rest. \textit{Right:} simultaneously enforcing several bits of $y$ steers generation to the intersection of semantic manifolds.}
    \label{fig:attrib_control}
\end{figure}
\vspace{-0.4em}
\vspace{-0.4em}
\subsection{Attribute-controlled generation}\vspace{-0.2cm}
\label{subsec:attrib_control}
We also control specific attributes without fixing the full semantic vector. Instead of a static $y$, at each step we compute $y_{\text{cur}} = g(\hat{x}_{0|t})$ and use a dynamic target that agrees with $y_{\text{cur}}$ except at the controlled index $n$, where $y_{\text{target}}[n] = \mu_n + 2\sigma_n$ from the empirical statistics. The latent-space update is then sparse, so the prior freely hallucinates the remaining attributes (Fig.~\ref{fig:attrib_control}, left); fixing several bits steers generation to the intersection of semantic manifolds (right).

\subsection{Compositional generation from object detection}\vspace{-0.2cm}
\label{subsec:detection_app}
We use object detection as the forward operator: given a real image with a known object, generate a new image whose detector output matches at the original bounding box (same class, same location) while the rest of the scene is freely sampled. The operator is a CenterNet-style head~\cite{zhou2019objects} on top of our ImageNet SPNN backbone, producing per-class center heatmaps, sub-pixel offsets, and box dimensions ($24\times 64\times 64$ at stride $4$). The prior is the class-conditional $256{\times}256$ ADM of~\cite{dhariwal2021diffusion}; per image we condition on the ImageNet class closest to the VOC class. NLBP is applied only for $t > 650$ (the noisier $\sim$35\% of the schedule), and after each NLBP step we Gaussian-mask the corrected $\hat{x}_{0|t}$ around the bounding box, so the constraint is enforced where the box is and the prior renders the rest. Fig.~\ref{fig:detection} shows that the object class and location are preserved while the rest of the scene is freshly sampled. Full pipeline in App.~\ref{app:detection}.
\begin{figure}
    \centering
    \includegraphics[width=0.66\linewidth]{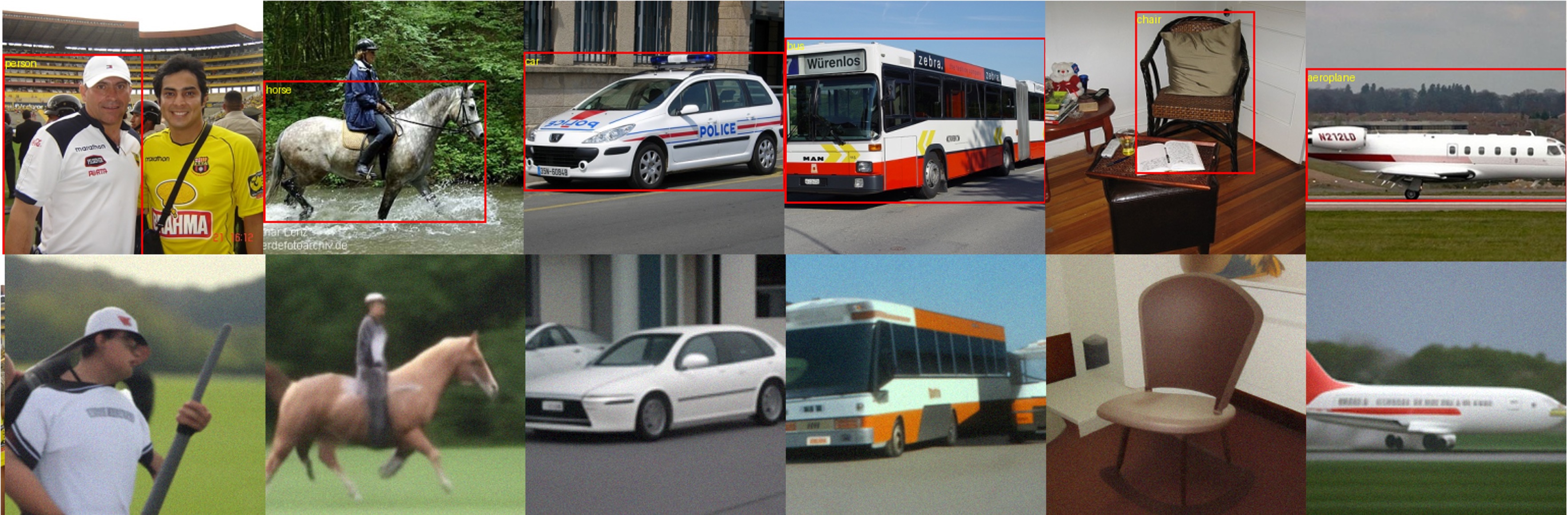}
    \caption{\small \textbf{Compositional generation from object detection.} \textit{Top:} VOC2007 test images with detection bounding boxes and class labels (person, horse, car, bus, chair, aeroplane). \textit{Bottom:} NLBP-guided samples projected onto the set of images that the detector maps to the same output (top): the object class and location are preserved while the rest of the scene is freshly sampled. \vspace{-0.6cm}}
    \label{fig:detection}
\end{figure}
\vspace{-0.4em}
\subsection{One-step generation as an SPNN-Linearizer}\vspace{-0.2cm}
\label{subsec:linearizer_exp}
We instantiate the SPNN-Linearizer of Proposition~\ref{prop:spnn_linearizer} as a one-step flow-matching generator on CelebA~\cite{berman2025said} (Linearizer). The forward $g$ is an SPNN; $A_t$ is a low-rank matrix from an MLP. After training, the $N$-step Euler trajectory collapses to $B = \prod_i (I + \Delta t\, A_{t_i})$, and generation is $\hat{x}_1 = g^\dagger(B\, g(x_0))$. We outperform the bijective baseline at every operating point (Table~\ref{tab:linearizer}); qualitative samples in Fig.~\ref{fig:celeba_samples}, with the largest gap at one-step ($325$ vs.\ $405$). The advantage stems from dimension reduction: $A_t$ acts on a $128$-dim latent rather than $12{,}288$-dim, shrinking from $101.5$M to $1.7$M parameters ($59\times$). Total parameters drop from $102$M to $74$M, $B$-build from $6.21$ to $0.14$\,s, and inference from $0.55$ to $0.34$\,ms.

\vspace{-4mm}
\begin{table}[h]
\centering
\caption{\small \textbf{SPNN-Linearizer on CelebA $64{\times}64$.} Generation quality (FID, left block) and efficiency of the linear core (right block) vs.\ the bijective baseline~\cite{berman2025said}. ``$N{\to}1$'' collapses an $N$-step Euler trajectory $B = \prod_i (I + \Delta t\, A_{t_i})$ into a single matrix and runs one forward/inverse SPNN pass at inference.}
\label{tab:linearizer}
\footnotesize
\setlength{\tabcolsep}{3.5pt}
\begin{tabular}{l ccccc @{\hskip 14pt} ccccc}
\toprule
& \multicolumn{5}{c}{\textbf{FID} $\downarrow$} & \multicolumn{5}{c}{\textbf{Efficiency} $\downarrow$} \\
\cmidrule(lr){2-6} \cmidrule(lr){7-11}
& 1 step & 10 steps & 100 steps & $100{\to}1$ & $1000{\to}1$
& Dim & $A_t$ params & Total params & Build $B$ & Inf./img \\
\midrule
Ours              & \textbf{325}   & \textbf{120.9} & \textbf{110.4} & \textbf{124.5} & \textbf{113.2}
                  & \textbf{128}   & \textbf{1.7M}  & \textbf{74.0M} & \textbf{0.14}\,s & \textbf{0.345}\,ms \\
Linearizer & 405 & 152 & 127 & 131 & 123
                  & 12{,}288 & 101.5M & 102.4M & 6.21\,s & 0.552\,ms \\
\bottomrule
\end{tabular}
\end{table}
\section{Limitations}\vspace{-0.2cm}
\label{sec:limitations}
Two caveats accompany our framework. First, Definition~\ref{def:bij_completion} requires $g$ to be a smooth surjective submersion. The SPNN architecture satisfies this by construction ($s>0$), but real-world degradations $\mathcal{D}$ that we model with $g$ (e.g., classifiers with saturating logits, hard quantization) are not themselves submersions; the guarantees apply to the surrogate $g \approx \mathcal{D}$, not to $\mathcal{D}$. Second, the orthogonality of NLBP holds only insofar as $r$ approximates the joint minimizer of $\|G(g^\dagger(y))\|^2$ across blocks (Remark~\ref{rem:r_needed}); a poorly-trained $r$ retains $g g^\dagger = \mathrm{Id}$ but loses the projection structure of Proposition~\ref{prop:orthogonal}.
\clearpage
\begin{ack}
A.S.\ is supported by the Chaya Career Advancement Chair.
\end{ack}
\bibliography{example_paper}

@article{feistel1973cryptography,
  title={Cryptography and computer privacy},
  author={Feistel, Horst},
  journal={Scientific American},
  volume={228},
  number={5},
  pages={15--23},
  year={1973}
}

@article{dinh2014nice,
  title={{NICE}: Non-linear Independent Components Estimation},
  author={Dinh, Laurent and Krueger, David and Bengio, Yoshua},
  journal={arXiv preprint arXiv:1410.8516},
  year={2014}
}

@inproceedings{dinh2017density,
  title={Density estimation using {R}eal {NVP}},
  author={Dinh, Laurent and Sohl-Dickstein, Jascha and Bengio, Samy},
  booktitle={ICLR},
  year={2017}
}

@inproceedings{kingma2018glow,
  title={{G}low: Generative Flow with Invertible 1x1 Convolutions},
  author={Kingma, Durk P and Dhariwal, Prafulla},
  booktitle={NeurIPS},
  year={2018}
}

@inproceedings{trockman2021orthogonalizing,
  title={Orthogonalizing Convolutional Layers with the {C}ayley Transform},
  author={Trockman, Asher and Kolter, J Zico},
  booktitle={ICLR},
  year={2021}
}

@inproceedings{nielsen2020survae,
  title={{S}ur{VAE} Flows: Surjections to Bridge the Gap between {VAE}s and Flows},
  author={Nielsen, Didrik and Jaini, Priyank and Hoogeboom, Emiel and Winther, Ole and Welling, Max},
  booktitle={NeurIPS},
  year={2020}
}

@article{gofer2023generalized,
  title={Generalized Inversion of Nonlinear Operators},
  author={Gofer, Eyal and Gilboa, Guy},
  journal={arXiv preprint arXiv:2111.10755},
  year={2023}
}

@inproceedings{kawar2022denoising,
  title={Denoising Diffusion Restoration Models},
  author={Kawar, Bahjat and Elad, Michael and Ermon, Stefano and Song, Jiaming},
  booktitle={NeurIPS},
  year={2022}
}

@inproceedings{wang2023zero,
  title={Zero-Shot Image Restoration Using Denoising Diffusion Null-Space Model},
  author={Wang, Yinhuai and Yu, Jiwen and Zhang, Jian},
  booktitle={ICLR},
  year={2023}
}

@inproceedings{kawar2021snips,
  title={{SNIPS}: Solving Noisy Inverse Problems Stochastically},
  author={Kawar, Bahjat and Vaksman, Gregory and Elad, Michael},
  booktitle={NeurIPS},
  year={2021}
}

@inproceedings{chung2023diffusion,
  title={Diffusion Posterior Sampling for General Noisy Inverse Problems},
  author={Chung, Hyungjin and Kim, Jeongsol and Mccann, Michael T and Klasky, Marc L and Ye, Jong Chul},
  booktitle={ICLR},
  year={2023}
}

@inproceedings{chung2022improving,
  title={Improving Diffusion Models for Inverse Problems using Manifold Constraints},
  author={Chung, Hyungjin and Sim, Byeongsu and Ryu, Dohoon and Ye, Jong Chul},
  booktitle={NeurIPS},
  year={2022}
}

@article{irani1991improving,
  title={Improving resolution by image registration},
  author={Irani, Michal and Peleg, Shmuel},
  journal={CVGIP: Graphical Models and Image Processing},
  volume={53},
  number={3},
  pages={231--239},
  year={1991}
}

@article{penrose1955generalized,
  title={A generalized inverse for matrices},
  author={Penrose, Roger},
  journal={Proceedings of the Cambridge Philosophical Society},
  volume={51},
  number={3},
  pages={406--413},
  year={1955},
  publisher={Cambridge University Press}
}

@article{berman2025said,
  title={Who Said Neural Networks Aren't Linear?},
  author={Berman, Nimrod and Hallak, Assaf and Shocher, Assaf},
  journal={arXiv preprint arXiv:2510.08570},
  year={2025}
}

@inproceedings{rezende2015variational,
  title={Variational Inference with Normalizing Flows},
  author={Rezende, Danilo Jimenez and Mohamed, Shakir},
  booktitle={ICML},
  year={2015}
}

@inproceedings{ardizzone2019analyzing,
  title={Analyzing Inverse Problems with Invertible Neural Networks},
  author={Ardizzone, Lynton and Kruse, Jakob and Rother, Carsten and K{\"o}the, Ullrich},
  booktitle={ICLR},
  year={2019}
}

@inproceedings{dhariwal2021diffusion,
  title={Diffusion Models Beat {GAN}s on Image Synthesis},
  author={Dhariwal, Prafulla and Nichol, Alex},
  booktitle={NeurIPS},
  year={2021}
}

@inproceedings{rombach2022ldm,
  title={High-Resolution Image Synthesis with Latent Diffusion Models},
  author={Rombach, Robin and Blattmann, Andreas and Lorenz, Dominik and Esser, Patrick and Ommer, Bj{\"o}rn},
  booktitle={CVPR},
  year={2022}
}

@article{podell2023sdxl,
  title={{SDXL}: Improving Latent Diffusion Models for High-Resolution Image Synthesis},
  author={Podell, Dustin and English, Zion and Lacey, Kyle and Blattmann, Andreas and Dockhorn, Tim and M{\"u}ller, Jonas and Penna, Joe and Rombach, Robin},
  journal={arXiv preprint arXiv:2307.01952},
  year={2023}
}

@inproceedings{karras2018celebahq,
  title={Progressive Growing of {GAN}s for Improved Quality, Stability, and Variation},
  author={Karras, Tero and Aila, Timo and Laine, Samuli and Lehtinen, Jaakko},
  booktitle={ICLR},
  year={2018}
}

@article{bolluyt2024psinn,
  title={Pseudoinvertible Neural Networks},
  author={Bolluyt, Elijah D. and Comaniciu, Cristina},
  journal={IEEE Transactions on Artificial Intelligence},
  volume={5},
  number={2},
  pages={602--612},
  year={2024},
  publisher={IEEE}
}

@article{beitler2019pie,
  title={PIE: Pseudo-Invertible Encoder},
  author={Beitler, Jan Jetze and Sosnovik, Ivan and Smeulders, Arnold},
  journal={arXiv preprint arXiv:2111.00619},
  year={2021}
}

@article{zhou2019objects,
  title   = {Objects as Points},
  author  = {Zhou, Xingyi and Wang, Dequan and Kr{\"a}henb{\"u}hl, Philipp},
  journal = {arXiv preprint arXiv:1904.07850},
  year    = {2019}
}

@article{everingham2010pascal,
  title   = {The {Pascal} Visual Object Classes ({VOC}) Challenge},
  author  = {Everingham, Mark and Van~Gool, Luc and Williams, Christopher K.~I. and Winn, John and Zisserman, Andrew},
  journal = {International Journal of Computer Vision},
  volume  = {88},
  number  = {2},
  pages   = {303--338},
  year    = {2010}
}

@article{russakovsky2015imagenet,
  title   = {{ImageNet} Large Scale Visual Recognition Challenge},
  author  = {Russakovsky, Olga and Deng, Jia and Su, Hao and Krause, Jonathan
             and Satheesh, Sanjeev and Ma, Sean and Huang, Zhiheng and
             Karpathy, Andrej and Khosla, Aditya and Bernstein, Michael
             and Berg, Alexander C. and Fei-Fei, Li},
  journal = {International Journal of Computer Vision (IJCV)},
  volume  = {115},
  number  = {3},
  pages   = {211--252},
  year    = {2015}
}
\bibliographystyle{plainnat}
\section*{NeurIPS Paper Checklist}

\begin{enumerate}

\item {\bf Claims}
    \item[] Question: Do the main claims made in the abstract and introduction accurately reflect the paper's contributions and scope?
    \item[] Answer: \answerYes{}
    \item[] Justification: The abstract and introduction accurately reflect the contributions: a natural non-linear pseudo-inverse via bijective completion, an architecture (SPNN) that implements it, and a Non-Linear Back-Projection used for zero-shot semantic inverse problems. See Sections~\ref{sec:nonlinear_pinv}, \ref{sec:method}, \ref{sec:application}, \ref{sec:experiments}.
    \item[] Guidelines:
    \begin{itemize}
        \item The answer \answerNA{} means that the abstract and introduction do not include the claims made in the paper.
        \item The abstract and/or introduction should clearly state the claims made, including the contributions made in the paper and important assumptions and limitations. A \answerNo{} or \answerNA{} answer to this question will not be perceived well by the reviewers. 
        \item The claims made should match theoretical and experimental results, and reflect how much the results can be expected to generalize to other settings. 
        \item It is fine to include aspirational goals as motivation as long as it is clear that these goals are not attained by the paper. 
    \end{itemize}

\item {\bf Limitations}
    \item[] Question: Does the paper discuss the limitations of the work performed by the authors?
    \item[] Answer: \answerYes{}
    \item[] Justification: Limitations are discussed in Section~\ref{sec:limitations}.
    \item[] Guidelines:
    \begin{itemize}
        \item The answer \answerNA{} means that the paper has no limitation while the answer \answerNo{} means that the paper has limitations, but those are not discussed in the paper. 
        \item The authors are encouraged to create a separate ``Limitations'' section in their paper.
        \item The paper should point out any strong assumptions and how robust the results are to violations of these assumptions (e.g., independence assumptions, noiseless settings, model well-specification, asymptotic approximations only holding locally). The authors should reflect on how these assumptions might be violated in practice and what the implications would be.
        \item The authors should reflect on the scope of the claims made, e.g., if the approach was only tested on a few datasets or with a few runs. In general, empirical results often depend on implicit assumptions, which should be articulated.
        \item The authors should reflect on the factors that influence the performance of the approach. For example, a facial recognition algorithm may perform poorly when image resolution is low or images are taken in low lighting. Or a speech-to-text system might not be used reliably to provide closed captions for online lectures because it fails to handle technical jargon.
        \item The authors should discuss the computational efficiency of the proposed algorithms and how they scale with dataset size.
        \item If applicable, the authors should discuss possible limitations of their approach to address problems of privacy and fairness.
        \item While the authors might fear that complete honesty about limitations might be used by reviewers as grounds for rejection, a worse outcome might be that reviewers discover limitations that aren't acknowledged in the paper. The authors should use their best judgment and recognize that individual actions in favor of transparency play an important role in developing norms that preserve the integrity of the community. Reviewers will be specifically instructed to not penalize honesty concerning limitations.
    \end{itemize}

\item {\bf Theory assumptions and proofs}
    \item[] Question: For each theoretical result, does the paper provide the full set of assumptions and a complete (and correct) proof?
    \item[] Answer: \answerYes{}
    \item[] Justification: All theoretical results (Definitions~\ref{def:bij_completion}, \ref{def:natural_pinv}; Propositions~\ref{prop:consistency}, \ref{prop:orthogonal}, \ref{prop:right_inverse}; Remarks~\ref{rem:existence}, \ref{rem:multiblock}) are stated with explicit assumptions and proved in the main text.
    \item[] Guidelines:
    \begin{itemize}
        \item The answer \answerNA{} means that the paper does not include theoretical results. 
        \item All the theorems, formulas, and proofs in the paper should be numbered and cross-referenced.
        \item All assumptions should be clearly stated or referenced in the statement of any theorems.
        \item The proofs can either appear in the main paper or the supplemental material, but if they appear in the supplemental material, the authors are encouraged to provide a short proof sketch to provide intuition. 
        \item Inversely, any informal proof provided in the core of the paper should be complemented by formal proofs provided in appendix or supplemental material.
        \item Theorems and Lemmas that the proof relies upon should be properly referenced. 
    \end{itemize}

    \item {\bf Experimental result reproducibility}
    \item[] Question: Does the paper fully disclose all the information needed to reproduce the main experimental results of the paper to the extent that it affects the main claims and/or conclusions of the paper (regardless of whether the code and data are provided or not)?
    \item[] Answer: \answerYes{}
    \item[] Justification: Full architecture, training setup, hyper-parameters, and sampling schedule are provided in Section~\ref{sec:method} and Appendices~\ref{app:nlbp_details}--\ref{app:hyperparams}.
    \item[] Guidelines:
    \begin{itemize}
        \item The answer \answerNA{} means that the paper does not include experiments.
        \item If the paper includes experiments, a \answerNo{} answer to this question will not be perceived well by the reviewers: Making the paper reproducible is important, regardless of whether the code and data are provided or not.
        \item If the contribution is a dataset and\slash or model, the authors should describe the steps taken to make their results reproducible or verifiable. 
        \item Depending on the contribution, reproducibility can be accomplished in various ways. For example, if the contribution is a novel architecture, describing the architecture fully might suffice, or if the contribution is a specific model and empirical evaluation, it may be necessary to either make it possible for others to replicate the model with the same dataset, or provide access to the model. In general. releasing code and data is often one good way to accomplish this, but reproducibility can also be provided via detailed instructions for how to replicate the results, access to a hosted model (e.g., in the case of a large language model), releasing of a model checkpoint, or other means that are appropriate to the research performed.
        \item While NeurIPS does not require releasing code, the conference does require all submissions to provide some reasonable avenue for reproducibility, which may depend on the nature of the contribution. For example
        \begin{enumerate}
            \item If the contribution is primarily a new algorithm, the paper should make it clear how to reproduce that algorithm.
            \item If the contribution is primarily a new model architecture, the paper should describe the architecture clearly and fully.
            \item If the contribution is a new model (e.g., a large language model), then there should either be a way to access this model for reproducing the results or a way to reproduce the model (e.g., with an open-source dataset or instructions for how to construct the dataset).
            \item We recognize that reproducibility may be tricky in some cases, in which case authors are welcome to describe the particular way they provide for reproducibility. In the case of closed-source models, it may be that access to the model is limited in some way (e.g., to registered users), but it should be possible for other researchers to have some path to reproducing or verifying the results.
        \end{enumerate}
    \end{itemize}

\item {\bf Open access to data and code}
    \item[] Question: Does the paper provide open access to the data and code, with sufficient instructions to faithfully reproduce the main experimental results, as described in supplemental material?
    \item[] Answer: \answerNo{}
    \item[] Justification: We use the public CelebA-HQ dataset and a standard pre-trained unconditional DDPM; training and evaluation protocols are fully documented. Code is being prepared for release on the project page.
    \item[] Guidelines:
    \begin{itemize}
        \item The answer \answerNA{} means that paper does not include experiments requiring code.
        \item Please see the NeurIPS code and data submission guidelines (\url{https://neurips.cc/public/guides/CodeSubmissionPolicy}) for more details.
        \item While we encourage the release of code and data, we understand that this might not be possible, so \answerNo{} is an acceptable answer. Papers cannot be rejected simply for not including code, unless this is central to the contribution (e.g., for a new open-source benchmark).
        \item The instructions should contain the exact command and environment needed to run to reproduce the results. See the NeurIPS code and data submission guidelines (\url{https://neurips.cc/public/guides/CodeSubmissionPolicy}) for more details.
        \item The authors should provide instructions on data access and preparation, including how to access the raw data, preprocessed data, intermediate data, and generated data, etc.
        \item The authors should provide scripts to reproduce all experimental results for the new proposed method and baselines. If only a subset of experiments are reproducible, they should state which ones are omitted from the script and why.
        \item At submission time, to preserve anonymity, the authors should release anonymized versions (if applicable).
        \item Providing as much information as possible in supplemental material (appended to the paper) is recommended, but including URLs to data and code is permitted.
    \end{itemize}

\item {\bf Experimental setting/details}
    \item[] Question: Does the paper specify all the training and test details (e.g., data splits, hyperparameters, how they were chosen, type of optimizer) necessary to understand the results?
    \item[] Answer: \answerYes{}
    \item[] Justification: Optimizer, learning rate, batch size, number of epochs, loss weights and sampling parameters are listed in Appendix~\ref{app:hyperparams}.
    \item[] Guidelines:
    \begin{itemize}
        \item The answer \answerNA{} means that the paper does not include experiments.
        \item The experimental setting should be presented in the core of the paper to a level of detail that is necessary to appreciate the results and make sense of them.
        \item The full details can be provided either with the code, in appendix, or as supplemental material.
    \end{itemize}

\item {\bf Experiment statistical significance}
    \item[] Question: Does the paper report error bars suitably and correctly defined or other appropriate information about the statistical significance of the experiments?
    \item[] Answer: \answerYes{}
    \item[] Justification: Quantitative comparisons in Table~\ref{tab:dps_comparison} are averaged over 6{,}000 generated samples of Celeb-HQ and 10{,}000 of ImageNet. Per-attribute reconstruction is reported over 100 test samples (Fig.~\ref{fig:attrib_analysis}, Appendix~\ref{app:attribs}).
    \item[] Guidelines:
    \begin{itemize}
        \item The answer \answerNA{} means that the paper does not include experiments.
        \item The authors should answer \answerYes{} if the results are accompanied by error bars, confidence intervals, or statistical significance tests, at least for the experiments that support the main claims of the paper.
        \item The factors of variability that the error bars are capturing should be clearly stated (for example, train/test split, initialization, random drawing of some parameter, or overall run with given experimental conditions).
        \item The method for calculating the error bars should be explained (closed form formula, call to a library function, bootstrap, etc.)
        \item The assumptions made should be given (e.g., Normally distributed errors).
        \item It should be clear whether the error bar is the standard deviation or the standard error of the mean.
        \item It is OK to report 1-sigma error bars, but one should state it. The authors should preferably report a 2-sigma error bar than state that they have a 96\% CI, if the hypothesis of Normality of errors is not verified.
        \item For asymmetric distributions, the authors should be careful not to show in tables or figures symmetric error bars that would yield results that are out of range (e.g., negative error rates).
        \item If error bars are reported in tables or plots, the authors should explain in the text how they were calculated and reference the corresponding figures or tables in the text.
    \end{itemize}

\item {\bf Experiments compute resources}
    \item[] Question: For each experiment, does the paper provide sufficient information on the computer resources (type of compute workers, memory, time of execution) needed to reproduce the experiments?
    \item[] Answer: \answerYes{}
    \item[] Justification: Per-image sampling time for each method is reported in Table~\ref{tab:dps_comparison}; training configuration is summarized in Appendix~\ref{app:hyperparams}.
    \item[] Guidelines:
    \begin{itemize}
        \item The answer \answerNA{} means that the paper does not include experiments.
        \item The paper should indicate the type of compute workers CPU or GPU, internal cluster, or cloud provider, including relevant memory and storage.
        \item The paper should provide the amount of compute required for each of the individual experimental runs as well as estimate the total compute. 
        \item The paper should disclose whether the full research project required more compute than the experiments reported in the paper (e.g., preliminary or failed experiments that didn't make it into the paper). 
    \end{itemize}
    
\item {\bf Code of ethics}
    \item[] Question: Does the research conducted in the paper conform, in every respect, with the NeurIPS Code of Ethics \url{https://neurips.cc/public/EthicsGuidelines}?
    \item[] Answer: \answerYes{}
    \item[] Justification: The research conforms to the NeurIPS Code of Ethics.
    \item[] Guidelines:
    \begin{itemize}
        \item The answer \answerNA{} means that the authors have not reviewed the NeurIPS Code of Ethics.
        \item If the authors answer \answerNo, they should explain the special circumstances that require a deviation from the Code of Ethics.
        \item The authors should make sure to preserve anonymity (e.g., if there is a special consideration due to laws or regulations in their jurisdiction).
    \end{itemize}

\item {\bf Broader impacts}
    \item[] Question: Does the paper discuss both potential positive societal impacts and negative societal impacts of the work performed?
    \item[] Answer: \answerYes{}
    \item[] Justification: Broader impacts - including risk of misuse of attribute-editing methods for synthetic media and the inheritance of dataset biases - are discussed in Appendix~\ref{app:impact}.
    \item[] Guidelines:
    \begin{itemize}
        \item The answer \answerNA{} means that there is no societal impact of the work performed.
        \item If the authors answer \answerNA{} or \answerNo, they should explain why their work has no societal impact or why the paper does not address societal impact.
        \item Examples of negative societal impacts include potential malicious or unintended uses (e.g., disinformation, generating fake profiles, surveillance), fairness considerations (e.g., deployment of technologies that could make decisions that unfairly impact specific groups), privacy considerations, and security considerations.
        \item The conference expects that many papers will be foundational research and not tied to particular applications, let alone deployments. However, if there is a direct path to any negative applications, the authors should point it out. For example, it is legitimate to point out that an improvement in the quality of generative models could be used to generate Deepfakes for disinformation. On the other hand, it is not needed to point out that a generic algorithm for optimizing neural networks could enable people to train models that generate Deepfakes faster.
        \item The authors should consider possible harms that could arise when the technology is being used as intended and functioning correctly, harms that could arise when the technology is being used as intended but gives incorrect results, and harms following from (intentional or unintentional) misuse of the technology.
        \item If there are negative societal impacts, the authors could also discuss possible mitigation strategies (e.g., gated release of models, providing defenses in addition to attacks, mechanisms for monitoring misuse, mechanisms to monitor how a system learns from feedback over time, improving the efficiency and accessibility of ML).
    \end{itemize}
    
\item {\bf Safeguards}
    \item[] Question: Does the paper describe safeguards that have been put in place for responsible release of data or models that have a high risk for misuse (e.g., pre-trained language models, image generators, or scraped datasets)?
    \item[] Answer: \answerNA{}
    \item[] Justification: The paper does not release high-risk pretrained models; the method relies on a standard pre-trained DDPM and a classifier, applied to the public CelebA-HQ dataset.
    \item[] Guidelines:
    \begin{itemize}
        \item The answer \answerNA{} means that the paper poses no such risks.
        \item Released models that have a high risk for misuse or dual-use should be released with necessary safeguards to allow for controlled use of the model, for example by requiring that users adhere to usage guidelines or restrictions to access the model or implementing safety filters. 
        \item Datasets that have been scraped from the Internet could pose safety risks. The authors should describe how they avoided releasing unsafe images.
        \item We recognize that providing effective safeguards is challenging, and many papers do not require this, but we encourage authors to take this into account and make a best faith effort.
    \end{itemize}

\item {\bf Licenses for existing assets}
    \item[] Question: Are the creators or original owners of assets (e.g., code, data, models), used in the paper, properly credited and are the license and terms of use explicitly mentioned and properly respected?
    \item[] Answer: \answerYes{}
    \item[] Justification: We cite all external assets used (CelebA-HQ~\cite{karras2018celebahq}, ImageNet~\cite{russakovsky2015imagenet}, Pascal~VOC~2007~\cite{everingham2010pascal}, DDPM backbone~\cite{dhariwal2021diffusion}, Stable Diffusion~1.5~\cite{rombach2022ldm}, SDXL~\cite{podell2023sdxl}, DPS~\cite{chung2023diffusion}) and use them under their respective public licenses and terms of use.
    \item[] Guidelines:
    \begin{itemize}
        \item The answer \answerNA{} means that the paper does not use existing assets.
        \item The authors should cite the original paper that produced the code package or dataset.
        \item The authors should state which version of the asset is used and, if possible, include a URL.
        \item The name of the license (e.g., CC-BY 4.0) should be included for each asset.
        \item For scraped data from a particular source (e.g., website), the copyright and terms of service of that source should be provided.
        \item If assets are released, the license, copyright information, and terms of use in the package should be provided. For popular datasets, \url{paperswithcode.com/datasets} has curated licenses for some datasets. Their licensing guide can help determine the license of a dataset.
        \item For existing datasets that are re-packaged, both the original license and the license of the derived asset (if it has changed) should be provided.
        \item If this information is not available online, the authors are encouraged to reach out to the asset's creators.
    \end{itemize}

\item {\bf New assets}
    \item[] Question: Are new assets introduced in the paper well documented and is the documentation provided alongside the assets?
    \item[] Answer: \answerNA{}
    \item[] Justification: No new dataset is released; model code will be made available along with a README documenting training and evaluation.
    \item[] Guidelines:
    \begin{itemize}
        \item The answer \answerNA{} means that the paper does not release new assets.
        \item Researchers should communicate the details of the dataset\slash code\slash model as part of their submissions via structured templates. This includes details about training, license, limitations, etc. 
        \item The paper should discuss whether and how consent was obtained from people whose asset is used.
        \item At submission time, remember to anonymize your assets (if applicable). You can either create an anonymized URL or include an anonymized zip file.
    \end{itemize}

\item {\bf Crowdsourcing and research with human subjects}
    \item[] Question: For crowdsourcing experiments and research with human subjects, does the paper include the full text of instructions given to participants and screenshots, if applicable, as well as details about compensation (if any)? 
    \item[] Answer: \answerNA{}
    \item[] Justification: No crowdsourcing or human-subject experiments were conducted.
    \item[] Guidelines:
    \begin{itemize}
        \item The answer \answerNA{} means that the paper does not involve crowdsourcing nor research with human subjects.
        \item Including this information in the supplemental material is fine, but if the main contribution of the paper involves human subjects, then as much detail as possible should be included in the main paper. 
        \item According to the NeurIPS Code of Ethics, workers involved in data collection, curation, or other labor should be paid at least the minimum wage in the country of the data collector. 
    \end{itemize}

\item {\bf Institutional review board (IRB) approvals or equivalent for research with human subjects}
    \item[] Question: Does the paper describe potential risks incurred by study participants, whether such risks were disclosed to the subjects, and whether Institutional Review Board (IRB) approvals (or an equivalent approval/review based on the requirements of your country or institution) were obtained?
    \item[] Answer: \answerNA{}
    \item[] Justification: No human-subjects research was conducted; all imagery is from the publicly released CelebA-HQ dataset.
    \item[] Guidelines:
    \begin{itemize}
        \item The answer \answerNA{} means that the paper does not involve crowdsourcing nor research with human subjects.
        \item Depending on the country in which research is conducted, IRB approval (or equivalent) may be required for any human subjects research. If you obtained IRB approval, you should clearly state this in the paper. 
        \item We recognize that the procedures for this may vary significantly between institutions and locations, and we expect authors to adhere to the NeurIPS Code of Ethics and the guidelines for their institution. 
        \item For initial submissions, do not include any information that would break anonymity (if applicable), such as the institution conducting the review.
    \end{itemize}

\item {\bf Declaration of LLM usage}
    \item[] Question: Does the paper describe the usage of LLMs if it is an important, original, or non-standard component of the core methods in this research? Note that if the LLM is used only for writing, editing, or formatting purposes and does \emph{not} impact the core methodology, scientific rigor, or originality of the research, declaration is not required.
    \item[] Answer: \answerNA{}
    \item[] Justification: LLMs were used only as a standard writing aid (proof-reading and rephrasing) and did not contribute to any core methodological or scientific content.
    \item[] Guidelines:
    \begin{itemize}
        \item The answer \answerNA{} means that the core method development in this research does not involve LLMs as any important, original, or non-standard components.
        \item Please refer to our LLM policy in the NeurIPS handbook for what should or should not be described.
    \end{itemize}

\end{enumerate}
\newpage
\appendix
\section{Detailed SPNN architecture}
\label{app:architecture}
This appendix gives the full per-stage channel configuration of our $256{\times}256$ SPNN classifiers. The network alternates \texttt{PixelUnshuffle} (factor $4$) stages with \texttt{ConvSPNN} coupling blocks. Table~\ref{tab:arch_detail} shows the configuration for the CelebA-HQ classifier ($40$ attribute logits): a single \texttt{ConvSPNN} block at each spatial scale. The ImageNet classifier (Table~\ref{tab:arch_detail_imagenet}) uses the same building blocks but stacks two \texttt{ConvSPNN} blocks at each of the early scales, gradually reducing the channel count in two halving steps rather than a single drop. This deeper schedule is needed to accommodate the much higher-dimensional output ($1000$ class logits vs.\ $40$ attribute logits) without overly aggressive per-block compression.
\begin{table}[!htbp]
\centering
\caption{\small Channel layout of our CelebA-HQ SPNN (input $3 \times 256 \times 256$). Each \texttt{PixelUnshuffle}($r$) reduces spatial resolution by $r$ and multiplies channels by $r^2$. Each \texttt{ConvSPNN}($c_{\text{in}} \to c_{\text{out}}$) retains $c_{\text{out}}$ ``content'' channels and discards $c_{\text{in}} - c_{\text{out}}$ ``null-space'' channels. All blocks use orthogonal mixing ($U$), hidden width $128$ for the internal \texttt{ConvMLP}, and $s$-activation $\exp(2\tanh(\cdot))$.}
\label{tab:arch_detail}
\small
\begin{tabular}{llll}
\toprule
Stage & Block & Output shape & Notes \\
\midrule
Input         & --                                    & $[3, 256, 256]$    & RGB image \\
Unshuffle (1) & \texttt{PixelUnshuffle}(4)            & $[48, 64, 64]$     & spatial $\times \tfrac{1}{4}$ \\
Block~1       & \texttt{ConvSPNN}($48 \to 12$)        & $[12, 64, 64]$     & U-Net ConvMLP ($128{\to}256$) \\
Unshuffle (2) & \texttt{PixelUnshuffle}(4)            & $[192, 16, 16]$    & \\
Block~2       & \texttt{ConvSPNN}($192 \to 48$)       & $[48, 16, 16]$     & U-Net ConvMLP ($256{\to}512$) \\
Unshuffle (3) & \texttt{PixelUnshuffle}(4)            & $[768, 4, 4]$      & \\
Block~3       & \texttt{ConvSPNN}($768 \to 192$)      & $[192, 4, 4]$      & U-Net ConvMLP ($512{\to}1024$) \\
Unshuffle (4) & \texttt{PixelUnshuffle}(4)            & $[3072, 1, 1]$     & collapse spatial \\
Block~4       & \texttt{ConvSPNN}($3072 \to 1024$)    & $[1024, 1, 1]$     & stack of $1{\times}1$ convs \\
Block~5       & \texttt{ConvSPNN}($1024 \to 40$)      & $[40, 1, 1]$       & output logits (40 attrs) \\
\bottomrule
\end{tabular}
\end{table}
\begin{table}[!htbp]
\centering
\caption{\small Channel layout of our ImageNet SPNN (input $3 \times 256 \times 256$). The same building blocks as the CelebA-HQ configuration (Table~\ref{tab:arch_detail}), but with two \texttt{ConvSPNN} blocks at each of the early spatial scales (gradual channel reduction by halves) rather than a single block per scale.}
\label{tab:arch_detail_imagenet}
\small
\begin{tabular}{llll}
\toprule
Stage & Block & Output shape & Notes \\
\midrule
Input         & --                                  & $[3, 256, 256]$    & RGB image \\
Unshuffle (1) & \texttt{PixelUnshuffle}(4)          & $[48, 64, 64]$     & spatial $\times \tfrac{1}{4}$ \\
Block~1       & \texttt{ConvSPNN}($48 \to 24$)      & $[24, 64, 64]$     & U-Net ConvMLP \\
Block~2       & \texttt{ConvSPNN}($24 \to 12$)      & $[12, 64, 64]$     & U-Net ConvMLP \\
Unshuffle (2) & \texttt{PixelUnshuffle}(4)          & $[192, 16, 16]$    & \\
Block~3       & \texttt{ConvSPNN}($192 \to 96$)     & $[96, 16, 16]$     & U-Net ConvMLP \\
Block~4       & \texttt{ConvSPNN}($96 \to 48$)      & $[48, 16, 16]$     & U-Net ConvMLP \\
Unshuffle (3) & \texttt{PixelUnshuffle}(4)          & $[768, 4, 4]$      & \\
Block~5       & \texttt{ConvSPNN}($768 \to 192$)    & $[192, 4, 4]$      & U-Net ConvMLP \\
Unshuffle (4) & \texttt{PixelUnshuffle}(4)          & $[3072, 1, 1]$     & collapse spatial \\
Block~6       & \texttt{ConvSPNN}($3072 \to 1024$)  & $[1024, 1, 1]$     & stack of $1{\times}1$ convs \\
Block~7       & \texttt{ConvSPNN}($1024 \to 1000$)  & $[1000, 1, 1]$     & output logits (1000 classes) \\
\bottomrule
\end{tabular}
\end{table}
\paragraph{\texttt{ConvMLP} (sub-network for $s$, $t$, $r$).}
For blocks where the feature map is larger than $1{\times}1$, $s$, $t$, and $r$ each use a U-Net-style \texttt{ConvMLP}: two $3{\times}3$ convolutions at the feature resolution, a stride-$2$ down-step, two $3{\times}3$ convolutions at the halved resolution, a stride-$2$ transposed-conv up-step, and a final $3{\times}3$ projection to the output channels. All activations are ReLU. The final projection is initialized to zero so the layer starts as identity. For the terminal $1{\times}1$ blocks we replace the U-Net with a $3$-layer $1{\times}1$ convolution stack ($c_{\text{in}} \to 2c_{\text{in}} \to 1.5 c_{\text{in}} \to c_{\text{out}}$, ReLU between layers, last layer zero-initialized).
\paragraph{Orthogonal mixing.}
Each \texttt{ConvSPNN} block is preceded by a learnable orthogonal $U \in \mathbb{R}^{C \times C}$ applied as a $1{\times}1$ convolution across channels. $U$ is parameterized via the Cayley transform $U = (I - A)(I + A)^{-1}$ for a learned skew-symmetric $A = B - B^\top$, which guarantees $U^\top U = I$ by construction. A Householder parameterization ($U = H_k \cdots H_1$, $H_i = I - 2 v_i v_i^\top / \|v_i\|^2$) is also provided as an alternative.
\paragraph{Training-loss weights.}
Phase I uses BCE classification loss ($\lambda_{\text{task}} = 1$), right-inverse loss ($\lambda_{\text{surj}} = 40$), and image-reconstruction stability loss ($\lambda_{\text{stab}} = 40$). Phase II freezes the forward path and optimizes only $r$ with $\lambda_{\text{natural}} = 0.3$, $\lambda_{r\_\text{surj}} = 1$, $\lambda_{r\_\text{stab}} = 1$. Full training hyper-parameters appear in the next appendix (Tables~\ref{tab:spnn_config}, \ref{tab:diffusion_config}).
\clearpage
\section{Architecture and hyper-parameters}
\label{app:hyperparams}
\begin{table}[!htbp]
\centering
\caption{\small SPNN-Celeb-HQ architecture and training configuration.}
\label{tab:spnn_config}
\small
\begin{tabular}{ll}
\toprule
\textbf{Category} & \textbf{Setting} \\
\midrule
\multicolumn{2}{l}{\textbf{Model architecture}} \\
Input resolution            & $256 \times 256$ RGB \\
Number of SPNN blocks       & 5 \\
Orthogonal mixing           & Cayley $1{\times}1$ convolutions \\
Latent partitioning         & Channel-wise \\
Activation (scale)          & $\tanh(\cdot) \times 2.0$, then $\exp$ \\
$s$, $t$, $r$ networks      & \texttt{ConvMLP}, hidden width 128 (see Appendix~\ref{app:architecture}) \\
Number of attributes        & 40 (CelebA-HQ) \\
\midrule
\multicolumn{2}{l}{\textbf{Training setup}} \\
Dataset                     & CelebA-HQ \\
Batch size                  & 256 \\
Optimizer                   & Adam \\
Learning rate               & $2 \times 10^{-4}$ \\
Adam $(\beta_1, \beta_2)$   & $(0.9, 0.999)$ \\
Epochs (Phase I)            & 15 \\
Gradient clipping           & 1.0 \\
Warmup iterations           & 200 \\
Precision                   & FP32 \\
Random seed                 & 556 \\
\midrule
\multicolumn{2}{l}{\textbf{Loss weights}} \\
$\lambda_{\text{task}}$     & 1.0 \\
$\lambda_{\text{surj}}$ (Phase I)   & 40.0 \\
$\lambda_{\text{stab}}$ (Phase I)   & 40.0 \\
$\lambda_{\text{natural}}$ (Phase II) & 0.3 \\
$\lambda_{r\_\text{surj}}$ (Phase II) & 1.0 \\
$\lambda_{r\_\text{stab}}$ (Phase II) & 1.0 \\
$r$-optimization epochs       & 50 \\
$r$-optimization learning rate & $1 \times 10^{-4}$ \\
$r$-optimization batch size   & 256 \\
\bottomrule
\end{tabular}
\end{table}
\begin{table}[!htbp]
\centering
\caption{\small SPNN-ImageNet architecture and training configuration.}
\label{tab:spnn_config_imagenet}
\small
\begin{tabular}{ll}
\toprule
\textbf{Category} & \textbf{Setting} \\
\midrule

\multicolumn{2}{l}{\textbf{Model architecture}} \\
Input resolution & $256 \times 256$ RGB \\
Number of \texttt{ConvSPNN} blocks & 9 \\
Orthogonal mixing & Cayley $1{\times}1$ convolutions \\
Latent partitioning & Channel-wise \\
Activation (scale) & $\tanh(\cdot) \times 1.0$, then $\exp$ \\
$s$, $t$, $r$ networks & \texttt{ConvMLP}, hidden width 256 (see Appendix~\ref{app:architecture}) \\
Number of classes & 1000 (ImageNet ILSVRC-2012) \\

\midrule
\multicolumn{2}{l}{\textbf{Training setup}} \\
Dataset & ImageNet ILSVRC-2012~\cite{russakovsky2015imagenet} (1.28M train, 50K val) \\
Hardware & 8 GPUs (DistributedDataParallel, NCCL) \\
Batch size & 512 total (64 per GPU) \\
Data-loader workers & 64 \\
Optimizer & SGD \\
Learning rate (initial) & $0.1$ \\
Momentum & $0.9$ \\
Weight decay & $1 \times 10^{-4}$ \\
LR schedule & Cosine annealing with linear warmup \\
Warmup epochs & 5 \\
Epochs (Phase I) & 100 \\
Gradient clipping (max-norm) & $1.0$ \\
Precision & FP32 \\
Train augmentation & RandomResizedCrop($256$), RandomHorizontalFlip \\
Test augmentation & Resize($292$), CenterCrop($256$) \\
Normalization & ImageNet stats ($\mu{=}[0.485, 0.456, 0.406]$, $\sigma{=}[0.229, 0.224, 0.225]$) \\

\midrule
\multicolumn{2}{l}{\textbf{Loss weights}} \\
$\lambda_{\text{task}}$ & $1.0$ \\
$\lambda_{\text{surj}}$ (Phase I) & $5.0$ \\
$\lambda_{\text{stab}}$ (Phase I) & $5.0$ \\
$\lambda_{\text{natural}}$ (Phase II) & $0.3$ \\
$\lambda_{r\_\text{surj}}$ (Phase II) & $1.0$ \\
$\lambda_{r\_\text{stab}}$ (Phase II) & $1.0$ \\
$r$-optimization epochs & $50$ \\
$r$-optimization learning rate & $1 \times 10^{-4}$ \\
$r$-optimization batch size & $512$ \\

\bottomrule
\end{tabular}
\end{table}

\begin{table}
\centering
\caption{\small \textbf{SPNN-CenterNet (object detection) architecture and training configuration.}
We reuse the early blocks of our SPNN-ImageNet classifier (Table~\ref{tab:spnn_config_imagenet}) as the detection backbone and add a single deeper detection-head block; the
full chain $A$ is then fine-tuned end-to-end on Pascal~VOC~2007 (Phase~I), after which the $r$-nets are optimized via the natural-PInv objective with the forward path frozen
(Phase~II).}
\label{tab:spnn_centernet_config}
\small
\begin{tabular}{lp{8cm}}
\toprule
\textbf{Category} & \textbf{Setting} \\
\midrule

\multicolumn{2}{l}{\textbf{Model architecture}} \\
Input resolution & $256 \times 256$ RGB \\
Output & $24 \times 64 \times 64$ at stride 4 (20-class hmap + 2 offset + 2 box-size) \\
Backbone & 2 \texttt{ConvSPNN} blocks ($12{\to}8$ at $128{\times}128$, $32{\to}28$ at $64{\times}64$) \\
         & interleaved with \texttt{PixelUnshuffle}($r{=}2$); init from SPNN-ImageNet (Table~\ref{tab:spnn_config_imagenet}) \\
Detection head & 1 deep \texttt{ConvSPNN} block ($28{\to}24$ at $64{\times}64$), 3-level U-net $s/t/r$ ConvMLPs, hidden 128 \\
Orthogonal mixing & Cayley $1{\times}1$ convolutions \\
$s$ activation & $\tanh(\cdot) \times 1.0$, then $\exp$ \\

\midrule
\multicolumn{2}{l}{\textbf{Phase I (forward training: backbone + head + adapters jointly)}} \\
Dataset & Pascal~VOC~2007~\cite{everingham2010pascal} (trainval for training, test for evaluation) \\
Hardware & 8 GPUs (DistributedDataParallel, NCCL) \\
Batch size & 64 total (8 per GPU) \\
Data-loader workers & 4 per GPU \\
Optimizer & Adam \\
Learning rate (initial) & $5 \times 10^{-4}$ \\
LR schedule & MultiStepLR, $\gamma = 0.1$ at epochs $\{90, 120\}$ \\
Epochs & 250 \\
Validation interval & every 5 epochs \\
Augmentation & Standard CenterNet (random crop, flip, color jitter, photometric distortion) \\
Normalization & ImageNet stats, BGR channel order (CenterNet convention) \\

\midrule
\multicolumn{2}{l}{\textbf{Phase I loss weights}} \\
$\lambda_{\text{hmap}}$ (focal loss) & $1.0$ \\
$\lambda_{\text{regs}}$ (L1, sub-pixel offset) & $1.0$ \\
$\lambda_{\text{wh}}$ (L1, box dimensions) & $0.1$ \\
$\lambda_{\text{img\_rec}}$ ($\|g^\dagger(Z(g(x))) - x\|^2$, $Z$ = peak sparsifier, App.~\ref{app:detection}) & $10.0$ \\

\midrule
\multicolumn{2}{l}{\textbf{Phase II ($r$-net only, forward path frozen)}} \\
Trainable parameters & per-block $r$ networks (all of $\{s, t, U\}$ and adapters frozen) \\
Batch size & 64 total (8 per GPU, DDP) \\
Optimizer & Adam \\
Learning rate & $1 \times 10^{-3}$ \\
Epochs & 200 \\
Gradient clipping (max-norm) & $1.0$ \\
$\lambda_{r\_\text{norm}}$ ($\|z(g^\dagger(Z(g(x)))) - z(0)\|^2$, latent regularizer) & $1.0$ \\
$\lambda_{r\_\text{rec}}$ ($\|g^\dagger(Z(g(x))) - x\|^2$, image reconstruction) & $10.0$ \\

\bottomrule
\end{tabular}
\end{table}

\begin{table}[!htbp]
\centering
\caption{\small Diffusion model and sampling configuration.}
\label{tab:diffusion_config}
\small
\begin{tabular}{ll}
\toprule
\textbf{Category} & \textbf{Setting} \\
\midrule
\multicolumn{2}{l}{\textbf{Data configuration}} \\
Dataset                     & CelebA-HQ \\
Image resolution            & $256 \times 256$ \\
Channels                    & 3 (RGB) \\
Random horizontal flip      & Enabled \\
Rescaled input              & Enabled \\
Data-loader workers         & 32 \\
\midrule
\multicolumn{2}{l}{\textbf{Model architecture}} \\
Model type                  & Simple U-Net \\
Input / output channels     & 3 / 3 \\
Base channel width          & 128 \\
Channel multipliers         & $[1, 1, 2, 2, 4, 4]$ \\
Residual blocks per level   & 2 \\
Self-attention resolutions  & 16 \\
Dropout                     & 0.0 \\
Variance type               & FixedSmall \\
EMA / EMA decay             & Enabled / 0.999 \\
Resampling with convolution & Enabled \\
\midrule
\multicolumn{2}{l}{\textbf{Diffusion process}} \\
Noise schedule              & Linear \\
$\beta_{\text{start}}$, $\beta_{\text{end}}$ & $10^{-4}$, $2 \cdot 10^{-2}$ \\
Number of diffusion steps   & 1000 \\
\midrule
\multicolumn{2}{l}{\textbf{Sampling and time-travel parameters}} \\
Sampling batch size         & 1 \\
Sampling horizon            & 100 \\
Travel length               & 1 \\
Travel repeat (reconstruction) & 1 \\
Travel repeat (attribute editing) & 2 \\
\bottomrule
\end{tabular}
\end{table}

\clearpage
\section{NLBP implementation details}
\label{app:nlbp_details}
We implement NLBP within a standard diffusion sampling framework, adopting the structure of DDNM~\cite{wang2023zero}.
\paragraph{Diffusion backbone.}
We use a pre-trained unconditional DDPM~\cite{dhariwal2021diffusion} with $T = 1000$ timesteps. The reverse process follows the standard noise schedule $\beta_t$. At each step $t$, we estimate $\hat{x}_{0|t}$ via the denoising network and apply the NLBP correction before re-sampling $x_{t-1}$.
\paragraph{Reconstruction strategy.}
We do not apply NLBP from the very first sampling step: in the early steps the predicted image is too blurry and noisy and is out of distribution for the degradation network. Over the diffusion schedule (T=1000 timesteps, sampled at 100 strides), we start NLBP at step $800$ for semantic reconstruction and at step $500$ for attribute-controlled generation.
\paragraph{Adaptive $\lambda$.}
For single and multi-attribute editing (Fig.~\ref{fig:attrib_control}), we scale $\lambda$ dynamically with the discrepancy between current and target. Let $\delta_n = |y_{\text{cur}}[n] - y_{\text{target}}[n]|$. We set $\lambda_t = \alpha \cdot \tanh(\gamma \cdot \delta_n)$, yielding stronger guidance when the controlled attribute is far from the target and gentler updates as it converges.
\paragraph{Back-projection guidance scheduling.}
\emph{Semantic reconstruction:} BP guidance is applied throughout the entire reverse process ($t \in [1000, 0]$) to preserve global semantic structure.
\emph{Single-attribute editing:} BP guidance is restricted to $t \in [500, 0]$; at earlier timesteps ($t > 500$) the BP step size is set to zero, allowing the diffusion prior to establish global structure before attribute control kicks in.
\emph{Multi-attribute editing:} we widen the guidance window to $t \in [800, 0]$ to account for the increased semantic coupling between attributes.
\paragraph{Covariance adjustment.}
Naive modification of a single attribute (e.g.,\ forcing ``Beard'' to $1$) can create an adversarial target if correlated attributes (e.g.,\ ``Male'') are not adjusted accordingly. To maintain a plausible $y_{\text{target}}$, we utilize the empirical covariance matrix $\Sigma$ of the training attributes. When modifying index $n$ by $\Delta y_n$, we adjust all other indices $j \neq n$ according to
\begin{equation}
    \Delta y_j = \frac{\Sigma_{jn}}{\Sigma_{nn}} \Delta y_n.
\end{equation}
This keeps $y_{\text{target}}$ on the natural manifold of attribute correlations, avoiding out-of-distribution artifacts. In practice, we use two complementary strategies: the covariance-based adjustment above for highly correlated attributes, and sparse single-attribute modification for weakly correlated ones (e.g.,\ \textit{Eyeglasses}).
\clearpage
\section{Object-detection NLBP: pipeline and hyperparameters}
\label{app:detection}
This appendix details the compositional generation experiment of \S\ref{subsec:detection_app}. The setup combines (i) an SPNN-CenterNet detector as the forward operator $A$ and its built-in pseudo-inverse $A^\dagger$, (ii) a class-conditional ImageNet ADM as the diffusion prior, (iii) a time-travel DDIM sampler with NLBP enabled only for noisy timesteps, and (iv) a pixel-space Gaussian smooth-paste that localizes the constraint at the bounding box.
\paragraph{Detector ($A$, $A^\dagger$).}
We attach a CenterNet-style head~\cite{zhou2019objects} to the ImageNet SPNN backbone of Appendix~\ref{app:architecture}. The head outputs (i) a per-class center heatmap $\mathrm{hmap} \in \mathbb{R}^{20\times 64\times 64}$, (ii) sub-pixel center offsets $\mathrm{regs} \in \mathbb{R}^{2\times 64\times 64}$, and (iii) box width/height in feature-map units $\mathrm{wh} \in \mathbb{R}^{2\times 64\times 64}$, concatenated as $y = [\mathrm{hmap}, \mathrm{regs}, \mathrm{wh}] \in \mathbb{R}^{24\times 64\times 64}$ at stride $4$ ($256/4 = 64$). The head is built with householder-mixed invertible affine layers and a depth-$2$ \texttt{ConvMLP} of hidden width $128$; the backbone is unfrozen during detector training. The pseudo-inverse $A^\dagger$ is provided by the SPNN's auxiliary network $r$, trained via the natural-PInv objective of Section~\ref{subsec:training}.
\paragraph{Diffusion prior.}
We use the public class-conditional $256{\times}256$ ADM checkpoint of Dhariwal and Nichol~\cite{dhariwal2021diffusion}: $256$ channels, $4$ heads, attention at resolutions $\{32, 16, 8\}$, $2$ residual blocks per level, learned $\sigma$, scale-shift normalization, FP16. Each VOC class is mapped to its closest ImageNet-1k class for conditioning the prior; the mapping table is given in Table~\ref{tab:voc_to_imagenet}.
\paragraph{Sampling schedule.}
Time-travel DDIM~\cite{wang2023zero} with $T_{\text{sampling}} = 100$, travel length $10$, travel repeat $3$ ($\sim 430$ network evaluations per image). NLBP is enabled only for $t > 650$ (the noisier $\sim$35\% of the schedule), with guidance scale $\lambda = 1.0$; for $t \le 650$ we let the prior sample freely.
\paragraph{Smooth paste in pixel space.}
After each NLBP step we replace $\hat{x}'_{0|t}$ by a Gaussian-weighted mix of the corrected estimate $\hat{x}_{0|t}^{\text{NLBP}}$ and the un-corrected prior estimate $\hat{x}_{0|t}$:
\begin{equation}
\hat{x}'_{0|t}(p) = w(p)\, \hat{x}_{0|t}^{\text{NLBP}}(p) + (1 - w(p))\, \hat{x}_{0|t}(p),\qquad
w(p) = \alpha \exp\!\left(-\tfrac{1}{2}\, \big\|(p - c_B) / (s\, d_B)\big\|^2\right),
\end{equation}
where $c_B$ is the bounding-box center, $d_B$ is the box (height, width), $s = 0.5$ is the spatial smoothness in BB-units, and $\alpha = 1.0$. The Gaussian mask localizes the constraint at the bounding box and lets the prior render the rest of the scene without spatial conflicts.
\paragraph{Hyperparameters.}
Best configuration: $\texttt{min\_bp\_step} = 650$, $\lambda_2 = 1.0$, $\lambda_1 = 0.0$, $\texttt{bp\_every}=1$, $\texttt{post\_bp\_noise}=0$, $\texttt{ref\_paste\_smoothness}=0.5$, $\texttt{ref\_paste\_alpha}=1.0$, per-image VOC$\to$ImageNet class. A sweep over $\texttt{min\_bp\_step}\in\{300, 500, 550, 650, 750, 800\}$ confirmed $650$ as the sweet spot: lower values over-constrain (output too close to original); higher values fail to localize. Per-image generation takes $\sim 22$\,s on an RTX 5090.
\begin{table}[!htbp]
\centering
\caption{\small VOC class $\to$ ImageNet-1k class index used for conditioning the ADM prior (\S\ref{subsec:detection_app}).}
\label{tab:voc_to_imagenet}
\footnotesize
\setlength{\tabcolsep}{6pt}
\begin{tabular}{lll lll}
\toprule
VOC & idx & ImageNet & VOC & idx & ImageNet \\
\midrule
aeroplane    & 404 & airliner          & dog          & 207 & golden retriever \\
bicycle      & 671 & mountain bike     & horse        & 339 & sorrel \\
bird         &  14 & indigo bunting    & motorbike    & 670 & motor scooter \\
boat         & 814 & speedboat         & person       & 981 & ballplayer \\
bottle       & 898 & water bottle      & pottedplant  & 738 & pot \\
bus          & 779 & school bus        & sheep        & 348 & ram \\
car          & 817 & sports car        & sofa         & 831 & studio couch \\
cat          & 281 & tabby             & train        & 466 & bullet train \\
chair        & 765 & rocking chair     & tvmonitor    & 851 & television \\
cow          & 345 & ox                & diningtable  & 532 & dining table \\
\bottomrule
\end{tabular}
\end{table}
\section{Full quantitative results on attribute reconstruction}
\label{app:attribs}
Table~\ref{tab:attrib_recon} reports per-attribute reconstruction fidelity for the semantic-reconstruction experiment of Section~\ref{subsec:reconstruction}, sorted by binary agreement. The mean over all 40 attributes is $92.3\%$ binary agreement, with the highest-agreement attributes corresponding to clear visual structure (\textit{Eyeglasses}, \textit{Wearing\_Necktie}, \textit{Bald}) and the lowest-agreement ones corresponding to subjective or ambiguous features (\textit{Mouth\_Slightly\_Open}, \textit{Arched\_Eyebrows}, \textit{Wearing\_Earrings}), where the classifier itself is least confident.
\begin{table}[!htbp]
\centering
\caption{\small Attribute reconstruction fidelity on CelebA-HQ, sorted by binary agreement (descending). ``Prob diff'' is the mean absolute difference in probability space between the reconstructed and original attribute.}
\label{tab:attrib_recon}
\small
\begin{tabular}{rlrr}
\toprule
\textbf{\#} & \textbf{Attribute} & \textbf{Prob diff} & \textbf{Agreement (\%)} \\
\midrule
0  & Blurry               & 0.0001 & 100 \\
1  & Eyeglasses           & 0.0046 & 100 \\
2  & Wearing Necktie      & 0.0065 & 100 \\
3  & Bald                 & 0.0027 & 99 \\
4  & Gray Hair            & 0.0074 & 99 \\
5  & Sideburns            & 0.0086 & 98 \\
6  & Pale Skin            & 0.0144 & 98 \\
7  & Narrow Eyes          & 0.0157 & 98 \\
8  & Double Chin          & 0.0040 & 98 \\
9  & Mustache             & 0.0097 & 97 \\
10 & Wearing Hat          & 0.0158 & 97 \\
11 & Blond Hair           & 0.0304 & 97 \\
12 & Chubby               & 0.0106 & 96 \\
13 & Straight Hair        & 0.0058 & 95 \\
14 & Oval Face            & 0.0107 & 95 \\
15 & Goatee               & 0.0252 & 95 \\
16 & 5 o'Clock Shadow     & 0.0185 & 94 \\
17 & Receding Hairline    & 0.0129 & 94 \\
18 & Wearing Necklace     & 0.0157 & 94 \\
19 & Heavy Makeup         & 0.0443 & 93 \\
20 & Wearing Lipstick     & 0.0463 & 93 \\
21 & Male                 & 0.0499 & 93 \\
22 & Bags Under Eyes      & 0.0127 & 92 \\
23 & Black Hair           & 0.0247 & 92 \\
24 & Brown Hair           & 0.0292 & 92 \\
25 & Bangs                & 0.0357 & 92 \\
26 & No Beard             & 0.0501 & 91 \\
27 & Young                & 0.0340 & 89 \\
28 & Wavy Hair            & 0.0229 & 89 \\
29 & Smiling              & 0.0572 & 88 \\
30 & Big Lips             & 0.0084 & 87 \\
31 & Bushy Eyebrows       & 0.0249 & 87 \\
32 & High Cheekbones      & 0.0560 & 87 \\
33 & Pointy Nose          & 0.0208 & 86 \\
34 & Big Nose             & 0.0259 & 86 \\
35 & Rosy Cheeks          & 0.0566 & 85 \\
36 & Attractive           & 0.0292 & 83 \\
37 & Wearing Earrings     & 0.0454 & 82 \\
38 & Arched Eyebrows      & 0.0534 & 81 \\
39 & Mouth Slightly Open  & 0.0989 & 80 \\
\bottomrule
\end{tabular}
\end{table}
\section{Comparison to text-to-image priors}
\label{app:t2i_comparison}
As a non-comparable reference, we additionally evaluate against large pre-trained text-to-image priors -- Stable Diffusion 1.5~\cite{rombach2022ldm} and SDXL~\cite{podell2023sdxl}, prompted with a natural-language description constructed from the target attribute vector (Table~\ref{tab:t2i_comparison}). These models receive a different conditioning channel from a logit-vector measurement and are not trained for the 40-way attribute task, so the comparison is not direct; we report it for reference. Despite the looser conditioning, our method beats every T2I configuration in both FID and third-party attribute agreement.
\begin{table}[!htbp]
\centering
\caption{\small Reference comparison against text-to-image priors on CelebA-HQ ($6{,}000$ samples, $40$-dim attribute prompts). gs: classifier-free guidance scale.}
\label{tab:t2i_comparison}
\footnotesize
\setlength{\tabcolsep}{4pt}
\begin{tabular}{llcccc}
\toprule
Method & Guidance & FID$\downarrow$ & KID$\downarrow$ & Agree 3rd$\uparrow$ & s/img \\
\midrule
NLBP+SPNN (ours) & NLBP                & \textbf{45.1} & \textbf{.033} & \textbf{87.9\%} & 5.7 \\
SD~1.5           & prompt (gs{=}7.5)   & 78.08         & .057          & 81.5\%          & 0.9 \\
SD~1.5           & prompt (gs{=}9.5)   & 82.58         & .065          & 81.0\%          & 0.9 \\
SD~1.5           & prompt (gs{=}11.5)   & 86.53         & .072          & 80.9\%          & 0.9 \\
SDXL             & prompt (gs{=}7.5)   & 82.32         & .058          & 79.5\%          & 3.7 \\
SDXL             & prompt (gs{=}9.5)   & 84.95         & .060          & 78.3\%          & 3.7 \\
SDXL             & prompt (gs{=}11.5)  & 88.59         & .061          & 77.9\%          & 3.7 \\
\bottomrule
\end{tabular}
\end{table}
\section{Qualitative Samples from the SPNN-Linearizer}
\label{app:samples}
Figure~\ref{fig:celeba_samples} shows random samples generated by the SPNN-Linearizer trained on CelebA-64, using one-step generation ($100{\to}1$ steps collapsed via $B$, FID: 124.5).
\begin{figure}[h]
    \centering
    \includegraphics[width=0.75\linewidth]{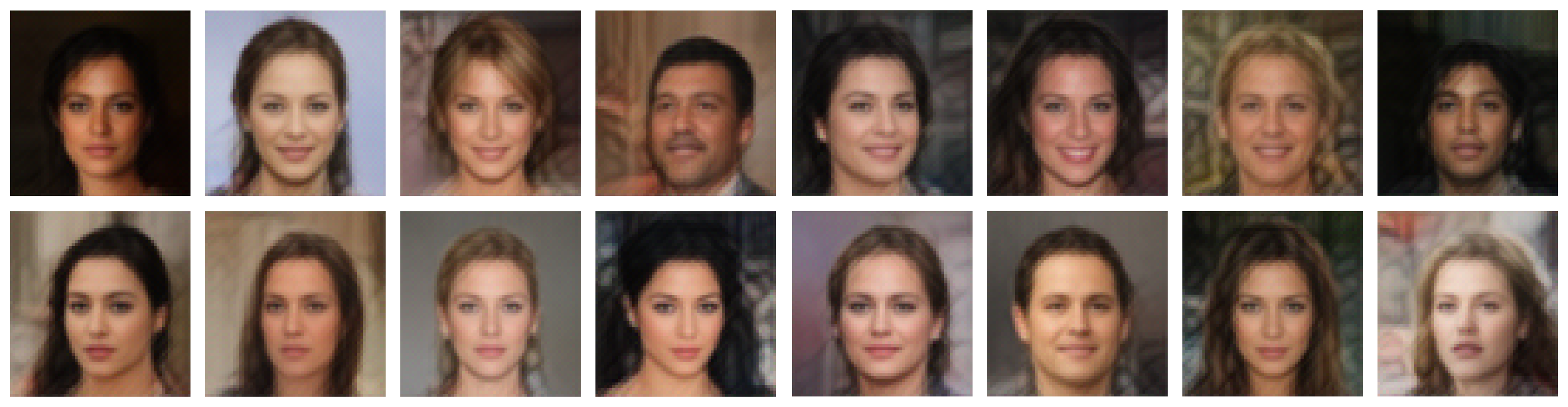}
    \caption{\small \textbf{Generated faces.} Random samples from the SPNN-Linearizer on CelebA $64{\times}64$, one-step generation ($100{\to}1$), FID\,=\,124.5.}
    \label{fig:celeba_samples}
\end{figure}
\section{Broader impact}
\label{app:impact}
This paper presents a theoretical framework for defining and computing non-linear pseudo-inverses in deep neural networks. The primary goal of this work is to advance the mathematical foundations of inverse problems, offering rigorous tools for ill-posed challenges in scientific computing, medical imaging, and physics where non-linear measurement operators are prevalent.
We demonstrate the method through semantic restoration and attribute editing on facial images. We acknowledge that techniques for controlling generative models, such as the one proposed here, carry potential risks related to the creation of synthetic media and deepfakes. While our method is designed to enforce consistency with semantic constraints, the ability to modify attributes or reconstruct plausible identities from abstract descriptors could in principle be misused for image manipulation.
Furthermore, our definition of the natural pseudo-inverse relies on learning the geometry of the training data. Consequently, any biases present in the training dataset (e.g.,\ correlations between demographic attributes and visual features) will be reflected in the reconstructed solutions. Practitioners applying this framework should be aware that the ``canonical'' solution provided by the model is only as neutral as the data distribution it was trained on.
\end{document}